%% file: uai2022-template.tex
\documentclass[accepted]{uai2022} 

\usepackage[american]{babel}

\usepackage{natbib} 
    \renewcommand{\bibsection}{\subsubsection*{References}}
\usepackage{mathtools} 
\usepackage{booktabs} 
\usepackage{tikz} 

\usepackage{times}
\usepackage{xcolor}

\usepackage{hyperref}
\usepackage{url}

\hypersetup{
    colorlinks=true,
    linkcolor=blue,
    citecolor=cyan,
    filecolor=green,      
    urlcolor=black,
}

\usepackage{algorithm, algpseudocode}
\usepackage{amsfonts,amssymb, amsmath, amsthm}

\usepackage{microtype}
\usepackage{graphicx}
\usepackage{subfigure}
\usepackage{booktabs} 

\usepackage{wrapfig}
\usepackage{bm}

\usepackage{enumitem}
\usepackage{xspace}
\usepackage{tcolorbox}

\input{header}

\usepackage{Definitions}

\newcommand{\AlgName}{{Spectral Dynamics Embedding}\xspace}
\newcommand{\algabb}{{SPEDE}\xspace}
\newcommand*{\dif}{\mathop{}\!\mathrm{d}}

\usepackage{xcolor}

\allowdisplaybreaks

\title{A Free Lunch from the Noise:\\
Provable and Practical Exploration for Representation Learning}

\author[1, 2, $^\star$]{\href{mailto:<tongzheng@utexas.edu>?Subject=Your UAI 2022 paper}{Tongzheng Ren }{}}
\author[3, $^\star$]{\href{mailto:<tianjunz@berkeley.edu>?Subject=Your UAI 2022 paper}{Tianjun Zhang }{}}
\author[4, 5]{Csaba Szepesv\'{a}ri~}
\author[2]{\href{mailto:<bodai@google.com>?Subject=Your UAI 2022 paper}{Bo Dai}{}
}
\affil[1]{%
    Department of Computer Science, UT Austin
}
\affil[2]{%
    Google Research, Brain Team
}
\affil[3]{
    Department of EECS, UC Berkeley
  }
\affil[4]{DeepMind}
\affil[5]{Department of Computer Science, University of Alberta}

\begin{document}
\maketitle

\begin{abstract}
Representation learning lies at the heart of the empirical success of deep learning for dealing with the curse of dimensionality. However, the power of representation learning has not been fully exploited yet in reinforcement learning (RL), due to {\bf i)}, the trade-off between expressiveness and tractability; and {\bf ii)}, the coupling between exploration and representation learning. In this paper, we first reveal the fact that under some noise assumption in the stochastic control model, we can obtain the linear spectral feature of its corresponding Markov transition operator in closed-form \emph{for free}. Based on this observation, we propose~\emph{\AlgName~(\algabb)}, which breaks the trade-off and completes optimistic exploration for representation learning by exploiting the structure of the noise. We provide rigorous theoretical analysis of~\algabb, and demonstrate the practical superior performance over the existing state-of-the-art empirical algorithms on several benchmarks. 
\end{abstract}

\let\thefootnote\relax\footnotetext{$^\star$ Equal Contribution}

\input{tex/intro}

\input{tex/preliminary}
\input{tex/algorithm}
\input{tex/analysis}
\input{tex/experiments}
\input{tex/conclusion}
\input{tex/acknowledgement}

\vspace{-5mm}
\bibliography{ref}

\newpage
\appendix
\onecolumn
\input{tex/appendix_background}
\input{tex/appendix_ucb}
\input{tex/appendix_proof}
\input{tex/appendix_complexity}
\input{tex/appendix_experiment}

\end{document}

%% file: header.tex
\usepackage{Definitions}
\usepackage{color}

%% file: tex/intro.tex

\section{Introduction}


Reinforcement learning~(RL) dedicates to solve the sequential decision making problem, where an agent is interacting with an \emph{unknown} environment to find the best policy that maximizes the expected cumulative rewards~\citep{sutton2018reinforcement}. It is known that the tabular algorithms direct controlling over the original state and action in Markov decision processes~(MDPs) achieve the minimax-optimal regret depending on the cardinality of the state and action space~\citep{jaksch2010near, azar2017minimax,jin2018q}. However, these algorithms become computationally intractable for the real-world problems with an enormous number of states.
Learning with function approximations upon \emph{good} representation is a natural idea to tackle such computational issue, which has already demonstrated its effectiveness in the success of deep learning~\citep{bengio2013representation}. In fact, representation learning lies at the heart of the empirical successes of deep RL in video games~\citep{mnih2013playing}, robotics~\citep{levine2016end}, Go~\citep{silver2017mastering} to name a few. Meanwhile, the importance and benefits of the representation in RL is rigorously justified~\citep{jin2020provably,yang2020reinforcement}, which quantifies the regret in terms of the dimension of the \emph{known} representation based on a subclass in MDPs~\citep{puterman2014markov}. A natural question raises:
\begin{center}
\emph{How to design {\bf provably efficient} and {\bf practical} algorithm for representation learning in RL?}
\end{center}
Here, by ``provably efficient'' we mean the sample complexity of the algorithm can be rigorously characterized only in terms of the complexity of representation class, without explicit dependency on the number of states and actions, while by ``practical'' we mean the algorithm can be implemented and deployed for the real-world applications. Therefore, we not only require the representation learned is expressive enough for handling complex practical environments, but also require the operations in the algorithm tractable and computation/memory efficient.    
The major difficulty of this question lies in two-fold: 
\begin{itemize}
    \item[{\bf i)}] The \emph{trade-off} between the expressiveness and the tractability in the design of the representations;

    \item[{\bf ii)}] The learning of representation is intimately \emph{coupled} with exploration.
\end{itemize}

Specifically, a desired representation should be sufficiently expressive\footnote{For a formal definition of expressiveness, see \citep{agarwal2020flambe}.} to capture the practical dynamic systems, while still computationally tractable. However, in general, expressive representation leads to complicated optimization in learning. 
For example, the representation in the linear MDP is \emph{exponential stronger} than the latent variable MDPs in terms of expressiveness~\citep{agarwal2020flambe}. However, its representation learning depends on either a MLE oracle that is computationally intractable due to the constraint on the regularity of conditional density~\citep{agarwal2020flambe}, or an optimization oracle that can solve complicated constrained $\min$-$\max$-$\min$-$\max$ optimization~\citep{modi2021model}. On the other hand,~\citet{misra2020kinematic} considers the representation introduced by an encoder in block MDP~\citep{du2019provably}, in which the learning problem can be completed by a regression, but with the payoff that the representations in block MDP is even weaker than the latent variable MDP~\citep{agarwal2020flambe}. 

Meanwhile, the coupling of the representation learning and exploration also induces the difficulty in \emph{practical} algorithm design and analysis. Specifically, one cannot learn a precise representation without enough experiences from a comprehensive exploration, while the exploration depends on a reliable estimation of the representation. Most of the known results depends on a policy-cover-based exploration~\citep{du2019provably,misra2020kinematic,agarwal2020flambe,modi2021model}, which maintains and samples a set of policies during training for systematic exploration, that significantly increases the computation and memory cost in implementation.

In this work, we propose~\emph{\AlgName~(\algabb)}, 
dealing with the aforementioned difficulties appropriately, and thus, answering the question affirmatively.~\algabb is established on a connection between the stochastic control dynamics~\citep{osband2014model,kakade2020information} and linear MDPs in~\secref{sec:algorithm}. Specifically, by exploiting the property of the \emph{noise} in the stochastic control dynamics, we can recover the factorization of its corresponding Markov transition operator in closed-form \emph{without extra computation}. This equivalency immediately overcomes the computational intractability in the linear MDP estimation via the corresponding control dynamics form, and thus, breaks the trade-off between expressiveness and tractability. Meanwhile, as a byproduct, the linear MDP reformulation also introduce efficient planning for optimal policy in nonlinear control through the linear sufficient feature from the spectral space of Markov operator, while in most model-based RL, planning is conducted by treating learned model as simulator, and thus, is inefficient and sub-optimal.


More importantly, the two faces of one model also provide the opportunity to tackle the coupling between representation learning and exploration. The optimism in the face of uncertainty principle can be easily implemented through Thompson sampling w.r.t. the stochastic nonlinear dynamics, which leads to the posterior of representations implicitly, while bypasses the unidentifiability issue in directly characterizing the representation, therefore, can be theoretically justified. 

We rigorously characterize the statistical property of~\algabb in terms of regret w.r.t. the complexity of representation class in~\secref{sec:analysis}, without explicit dependence on the size of raw state space and action space. With the established unified view, our results generalize online control~\citep{kakade2020information} and linear MDP~\citep{jin2020provably} beyond \emph{known} features. We finally demonstrate the superiority of~\algabb on the MuJoCo benchmarks in~\secref{sec:experiments}. It significantly outperforms the empirical state-of-the-art RL algorithms. To our knowledge, \algabb is the first representation learning algorithm achieving statistical, computational, and memory efficiency with sufficient expressiveness.  
 
\input{tex/related_work}

%% file: tex/related_work.tex

\subsection{Related Work}


There have been many great attempts on {\bf algorithmic representation learning} in RL for different purposes, \eg, bisimulation~\citep{ferns2004metrics,gelada2019deepmdp}, reconstruction~\citep{hafner2019learning}.
Recently, there are also several works considering the spectral features based on decomposing different variants of the transition operator, including successor features~\citep{dayan1993improving,kulkarni2016deep}, proto-value functions~\citep{mahadevan2007proto,wu2018laplacian}, spectral state-aggregation~\citep{duan2018state,zhang2019spectral}, and contrastive fourier features~\citep{nachum2021provable}. These works are highly-related to the proposed~\algabb. Besides these features focus on \emph{state-only} representation, the major differences between~\algabb and these spectral features lie in {\bf i)}, the target operators in existing spectral features are \emph{state-state} transition, which cancel the effect of action; {\bf ii)}, the target operators are estimated based on empirical data from a \emph{fixed behavior policy} under the implicit assumption that the estimated operator is \emph{uniformly accurate}, ignoring the major difficulty in exploration, while \algabb carefully designed the systematic exploration with theoretical guarantee; {\bf iii)}, most of the existing spectral features rely on \emph{explicitly} decomposition of the operators, while \algabb obtains the spectral \emph{for free}.  


Turning to the {\bf theoretically-justified representation learning with online exploration}, a large body of effort focuses on the policy-cover-based exploration~\citep{du2019provably,misra2020kinematic,agarwal2020flambe,modi2021model}. 
The major difficulty impedes their practical application is the computation and memory cost: the policy-cover-based exploration requires a set of exploratory polices to be maintained and sampled from during training, which can be extremely expensive. 
\citet{uehara2021representation} introduced a UCB mechanism that can enforce exploration without the requirements on maintaining the policy cover. However, the algorithm requires an MLE oracle for unnormalized conditional statistical model, which still prevents us from applying the algorithm in practice until recent attempt~\citep{zhang2022making} using contrastive learning to replace the MLE. 

Another two related lines of research are {\bf model-based RL} and {\bf online control}, which are commonly known overlapped but separate communities considering different formulations of the dynamics. Our finding bridges these two communities by establishing the equivalency between standard models that are widely considered in the corresponding communities.~\citet{osband2014model} and~\citet{kakade2020information} are the most related to our work in each community. These models generalize their corresponded linear models, \ie,~\citet{jin2020provably} and~\citet{cohen2019learning}, with general nonlinear model and kernel function within a known RKHS, respectively. The regret of the optimistic (pessimistic) algorithm has been carefully characterized for these models. However, both of the proposed algorithms in~\citet{osband2014model} and~\citet{kakade2020information} require a planning oracle to seek the optimal policy, which might be computationally intractable. In~\algabb, this is easily handled in the equivalent linear MDP. 


%% file: tex/preliminary.tex
\section{Preliminaries}

Markov Decision Process (MDP) is one of the most standard models studied in the reinforcement learning that can be denoted by the tuple $\mathcal{M} = (\mathcal{S}, \mathcal{A}, r, T, \rho, H)$, where $\mathcal{S}$ is the state space, $\mathcal{A}$ is the action space, $r:\mathcal{S} \times \mathcal{A} \to \mathbb{R}^+$ is the reward function\footnote{In general, the reward can be stochastic. Here for simplicity we assume the reward is deterministic and known throughout the paper, which is a common assumption in the literature \citep[\eg,][]{jin2018q, jin2020provably, kakade2020information}.} (where $\mathbb{R}^+$ denotes the set of non-negative real numbers), $T:\mathcal{S}\times \mathcal{A} \to \Delta(\mathcal{S})$ is the transition and $\rho$ is an initial state distribution and $H$ is the horizon\footnote{Our method can be generalized to infinite horizon case, see Section \ref{sec:practical_alg} for the detail.} (\ie, the length of each episode). A (potentially non-stationary) policy $\pi$ can be defined as $\{\pi_h\}_{h\in [H]}$ where $\pi_h:\mathcal{S}\to\Delta(\mathcal{A}), \forall h\in [H]$. Following the standard notation, we define the value function $V_h^\pi(s_h) := \mathbb{E}_{T, \pi}\left[\sum_{t=h}^{H-1} r(s_t, a_t)|s_h = s\right]$ and the action-value function (\ie, the $Q$ function) $Q_h^\pi(s_h, a_h) = \mathbb{E}_{T, \pi}\left[\sum_{t=h}^{H-1} r(s_t, a_t)|s_h = s, a_h = a\right]$, which are the expected cumulative rewards under transition $T$ when executing policy $\pi$ starting from $s_h$ and $(s_h, a_h)$. With these two definitions at hand, it is straightforward to show the following Bellman equation:
\begin{align*}
    Q_{h}^\pi(s_h, a_h) = r(s_h, a_h) + \mathbb{E}_{s_{h+1}\sim T(\cdot|s_h, a_h)} \sbr{V_{h+1}^\pi(s_{h+1})}.
\end{align*}
Most of RL algorithms aim at finding the optimal policy $\pi^* = \mathop{\arg\max}_{\pi} \mathbb{E}_{s\sim \rho} \sbr{V_0^\pi(s)}$ under MDPs. It is well known that in the tabular setting when the state space and action space are finite, we can provably identify the optimal policy with both sample-efficient and computational-efficient optimism-based methods \citep[\eg][]{azar2017minimax} with the complexity proportion to $\mathrm{poly}(|\mathcal{S}|, |\mathcal{A}|)$. However, in practice, the cardinality of state and action space can be large or even infinite. Hence, we need to incorporate function approximation into the learning algorithm when we deal with such cases. The linear MDP \citep{yang2019sample, jin2020provably, yang2020reinforcement} or low-rank MDP \citep{agarwal2020flambe, modi2021model} is the most well-known MDP class that can incorporate linear function approximation with theoretical guarantee, thanks to the following assumption on the transition and reward:
\begin{align}
\label{eq:linear_transition}
    T(s^\prime|s, a) = \langle \phi(s, a),  \mu(s^\prime)\rangle_{\mathcal{H}},\quad r(s, a) = \langle \phi\rbr{s, a}, \theta \rangle_\Hcal,
\end{align}
where $\phi:\mathcal{S}\times \mathcal{A} \to \mathcal{H}$, $\mu:\mathcal{S}\to\mathcal{H}$ are two feature maps and $\mathcal{H}$ is a Hilbert space. The most essential observation for them is that, $Q_h^\pi(s, a)$ for any policy $\pi$ is linear w.r.t $\phi(s_h, a_h)$, due to the following observation \citep{jin2020provably}:
\begin{align}
    & Q_h^\pi(s_h, a_h)= r(s_h, a_h) + \int V_{h+1}^\pi(s_{h+1}) T(s_{h+1}|s_h, a_h) \dif s_{h+1}   \nonumber\\
    & = \left\langle \phi(s_h, a_h), \theta + \int V_{h+1}^\pi(s_{h+1}) \mu(s_{h+1}) \dif s_{h+1}\right\rangle_{\mathcal{H}}.
\label{eq:linear_Q}
\end{align}
Therefore, $\phi$ serves as a sufficient representation for the estimation of $Q_h^\pi$, that can provide uncertainty estimation with standard linear model analysis and eventually lead to sample-efficient learning when $\phi$ is fixed and known to the agent \citep[see Theorem 3.1 in][]{jin2020provably}.
However, we in general do not have such representations in advance\footnote{One exception is the tabular MDP, where we can choose $\phi:\mathcal{S}\times \mathcal{A}\to \mathbb{R}^{|\mathcal{S}||\mathcal{A}|}$ that each state-action pair has exclusive one non-zero element and $\mu:\mathcal{S}\to \mathbb{R}^{|\mathcal{S}||\mathcal{A}|}$ correspondingly defined to make \eqref{eq:linear_transition} hold.} and we need to learn the representation from the data, which constraints the applicability of the algorithms derived with fixed and known representation.


\paragraph{Remark (Model-based RL vs. RL with Representation):} 
We would like to emphasize that, although most of the existing representation learning methods need to learn the transition~\citep{du2019provably,misra2020kinematic,agarwal2020flambe,uehara2021representation}, RL with representation learning is related but perpendicular to the concept of model-based RL. 
The major difference lies in how to use the learned transition for planning (\ie~finding the optimal policy). 
In vanilla model-based RL methods~\citep[\eg,][]{sutton1990integrated,chua2018deep,kurutach2018model}, the learned transition is played as a simulator generating samples for policy improvement; while in representation-based RL, the representation is extracted from the learned transition to compose the policy explicitly, which is significantly efficient comparing to the model-based RL methods.

%% file: tex/algorithm.tex
\section{\AlgName}\label{sec:algorithm}
It is naturally to consider how to perform sample-efficient representation learning (and hence sample-efficient reinforcement learning) that satisfies \eqref{eq:linear_transition} in an online manner. The most straightforward idea is performing the maximum likelihood estimation (MLE) in the representation space \citep[\eg,][]{agarwal2020flambe}. Unfortunately, for general cases, such MLE is intractable, due to the constraints on the regularity of marginal distribution (\ie, $\langle \phi(s, a), \int_{s^\prime} \mu(s^\prime) \dif s^\prime\rangle = 1$) for all $(s, a)\in\mathcal{S}\times \mathcal{A}$. 
Moreover, even we can perform MLE for certain cases (for example, the block MDP), as the representation is estimated from the data, which can be inaccurate, most of the existing work apply the policy cover technique \citep{du2019provably, misra2020kinematic, agarwal2020flambe, modi2021model} to enforce exploration.
However, such procedures can be both computational and memory expensive when we need amounts of exploratory policy to guarantee the coverage of whole state space, which makes it not a practical choice. 

To overcome these issues, we introduce \AlgName~(\algabb), which leverages the noise structure to provide a simple but provable efficient and practical algorithm for representation learning in RL. We first introduce our key observation, which induces the equivalency between linear MDP and stochastic nonlinear control.


\subsection{Key Observation}
Our fundamental observation is that, the density of isotropic Gaussian distribution can be expressed as the inner product of two feature maps, thanks to the reproducing property and the random Fourier transform of the Gaussian kernel\footnote{We provide a brief review on the related definitions in Appendix \ref{sec:background}.} \citep{rahimi2007random}:

\begin{tcolorbox}[colback=cyan!5!white,colframe=cyan!75!black]
\begin{align}
    &\phi(x| \mu, \sigma^2 I) \propto  \exp\left(-\frac{\|x - \mu\|^2}{2\sigma^2}\right) \nonumber\\
    = & \langle k(x, \cdot), k(\mu, \cdot)\rangle_{\mathcal{H}} \quad \textit{(Reproducing Property)} \label{eq:reproducing_property}\\
    = &  \inner{\varphi(x, \omega, b)}{\varphi(\mu, \omega, b)}_{p(\omega, b)}\quad \textit{(Random Fourier)},\label{eq:random_feature}
\end{align}
\end{tcolorbox}
where $k(\cdot, \cdot)$ is the Gaussian kernel with bandwidth $\sigma$: $k(x, y) = \exp\left(-\frac{\|x - y\|_2^2}{2\sigma^2}\right)$, $\mathcal{H}$ is the Reproducing Kernel Hilbert Space (RKHS) associated with $k$, $\varphi(x, \omega, b) = \sqrt{2}\cos(\omega^\top x + b)$, $\langle f, g\rangle_p = \mathbb{E}_{p(x)}[f(x)g(x)]$ and $p(\omega, b) =\mathcal{N}(\omega; 0, 1/\sigma^2 I)\cdot \Ucal(b; [0, 2\pi])$ with $\Ncal$ and $\Ucal$ denoting Gaussian and Uniform distribution, respectively.

Consider the general transition dynamics,
\begin{align}
\label{eq:control_model}
    & s^\prime = f^*(s, a) + \epsilon,\quad \epsilon\sim \Ncal(0, \sigma^2), \\
    \text{or equivalently}& \quad T(s'|s, a)\propto \exp\rbr{-\frac{\nbr{s' - f^*(s, a)}^2}{2\sigma^2}},
\end{align}
which is a widely used setup in the empirical model-based reinforcement learning \citep[\eg,][]{chua2018deep, kurutach2018model, clavera2018model, wang2019benchmarking}, and the online (non)-linear control \citep[\eg,][]{abbasi2011regret, mania2019certainty, mania2020active, simchowitz2020naive, kakade2020information}. Here $s\in\mathbb{R}^d$, $a\in\mathcal{A}$ that can be continuous and $f^*$ is a dynamic function. 

By applying the reproducing property~\eqref{eq:reproducing_property} or random Fourier transform~\eqref{eq:random_feature} for the transition dynamics~\eqref{eq:control_model}, we can obtain the feature $\phi$ and $\mu$ satisfies \eqref{eq:linear_transition} \emph{for free}. Specifically, taking the reproducing property as an example, we have that
\begin{align}
\label{eq:control_model_linear}
    T(s^\prime|s, a)  = \langle k(f^*(s, a), \cdot), (2\pi\sigma^2)^{-d/2} k(s^\prime, \cdot))\rangle_{\mathcal{H}},
\end{align}
which means the problem \eqref{eq:control_model} is indeed a linear MDP with $\phi(s, a) = k(f^*(s, a), \cdot)$ and $\mu(s^\prime) = (2\pi\sigma^2)^{-d/2} k(s^\prime, \cdot)$. Following \eqref{eq:linear_Q}, we know $Q(s, a)$ is in the linear span of the $\phi(s, a)$ that is transformed from $f^*(s, a)$. Therefore, finding a good representation of $Q(s, a)$ is equivalent to finding a good estimation of $f^*$. In the next section, we will show that, with the well-known optimism in the face of uncertainty (OFU) principle, we can estimate $f^*$ in an online manner with a both sample-efficient in terms of regret and computational-efficient algorithm.


\paragraph{Remark (Computation-free Factorizable Noise Model):} We remark that, similar observations also hold for large amounts of distributions, \eg, the Laplace and Cauchy distribution. We refer the interested reader to Table 1 in \citet{dai2014scalable} for the known transformation of kernels and features. Here we focus on the Gaussian noise.
\paragraph{Remark (Reward Factorization):} In the definition of linear MDP~\eqref{eq:linear_transition}, the reward function $r(s, a)$ should also have the ability to be linearly represented by $\phi\rbr{s, a}$. This can be implemented by augmenting $[\phi\rbr{s, a}, r(s, a)]$ as the new representation, therefore, we neglect the reward function throughout the paper. 


\subsection{Practical Algorithm Description}
\label{sec:practical_alg}
Here, we introduce a generic Thompson Sampling (TS) type algorithm in Algorithm \ref{alg:TS} based on the OFU principle that leverage our observation at the previous section.
At the beginning, we provide a prior distribution $\mathbb{P}(f)$ that reflects our prior knowledge on $f^*$. Then for each episode, we draw a $f$ from the posterior, find the optimal policy with $f$ using the planning algorithm, execute this policy and eventually inference the posterior with the new observation. Notice that, we choose the policy optimistically with an \emph{sampled} $f$, which enforces the exploration following the principle of OFU. Meanwhile, we only learn the dynamic with posterior inference {\color{black} and directly obtain the representation with \eqref{eq:reproducing_property} or \eqref{eq:random_feature}, which avoids additional error from the representation learning step}. As all of our data is collected with $f^*$, our posterior will shrink to a point mass of $f^*$, which guarantees we can identify good representation and good policy with sufficient number of data.

\begin{algorithm}[tb]
\caption{Thompson Sampling (TS) Algorithm}
\label{alg:TS}
\begin{algorithmic}[1]
\Require Number of Episodes $K$, Prior Distribution $\mathbb{P}(f)$, Reward Function $r(s, a)$.
\State Initialize the history set $\mathcal{H}_0 = \emptyset$.
\For{episodes $k=1, 2, \cdots$}
\State {\color{blue} Sample $f_k \sim \mathbb{P}(f|\mathcal{H}_k)$.} \Comment{Draw the Representation.}
\State {\color{blue} Find the optimal policy $\pi_k$ on $f_k$ with Algorithm \ref{alg:planning}.}\label{line:planning}\Comment{Planning with $f_k$.}
\For{steps $h=0, 1, \cdots, H-1$}\Comment{Executing $\pi_k$.}
\State Execute $a_h^k \sim \pi_k^h(s_h^k)$.
\State Observe $s_{h+1}$.
\EndFor
\State Set $\mathcal{H}_k = \mathcal{H}_{k-1} \cup \{(s_h^k, a_h^k, s_{h+1}^k)\}_{h=0}^{H-1}$. \Comment{Update the History.}
\EndFor
\end{algorithmic}
\end{algorithm}

One significant part of \algabb is the computational-efficient planning with $f_k$, thanks to the linear MDP formulation \eqref{eq:control_model_linear}. Prior work assumes an oracle \cite[\eg,][]{kakade2020information} for such planning problem, but little is known on how to provably perform such planning efficiently. Notice that, with the feature $\phi(s, a)$ defined via \eqref{eq:reproducing_property} and \eqref{eq:random_feature}, we know that $Q_h^\pi(s, a)$ is exactly linear in $\phi(s, a)$, $\forall h, \pi$. Hence, we can perform a dynamic programming style algorithm that calculates $Q_h^\pi(s, a)$ with the given feature $\phi(s, a)$, and then greedily select the action at each level $h$, which is simple yet efficient. It is straightforward to show that the policy obtained with this dynamic programming algorithm is optimal by induction. We illustrate the detailed algorithm in~\algtabref{alg:planning}.

\begin{algorithm}[tb]
\caption{Planning with Dynamic Programming}
\label{alg:planning}
\begin{algorithmic}[1]
\Require Transition Model $f$, Reward Function $r(s, a)$.
\State Initialize $\phi(s, a)$, $\mu(s^\prime)$ with \eqref{eq:reproducing_property} or \eqref{eq:random_feature}. $V_H(s) = 0, \forall s$.
\For{steps $h=H-1, H-2, \cdots, 0$}
\State {\color{blue} Compute
\begin{align*}
    Q_h(s, a) = r(s, a) + \langle \phi(s, a), \int V_{h+1}(s^\prime) \mu(s^\prime) \dif s^\prime\rangle_{\mathcal{H}}.
\end{align*}}\Comment{Bellman Update.}\label{line:critic}
\State Set $V_h(s) = \max_{a} Q_h(s, a)$, $\pi_h(s) = \mathop{\arg\max}_a Q_h(s, a)$.\Comment{Choose the Optimal Policy.}\label{line:actor}
\EndFor
\State \Return $\{\pi_h\}_{h=0}^{H-1}$.
\end{algorithmic}
\end{algorithm}



\subsubsection{Implementation Details}
In such a planning algorithm, we need to maintain the posterior of $f$ and calculate the term $\int V_{h+1}(s^\prime) \mu(s^\prime) \dif s^\prime$ and take the maximum of $Q_h(s, a)$ over $a$, which can be problematic. We will provide more discussion on this issue below.

\paragraph{Posterior Sampling} The exact posterior inference can be hard if $f^*$ does not lie in simple function class (\eg, linear function class) or has some derived property (\eg, conjugacy), so in practice we apply the existing mature approximate inference methods like Markov Chain Monte Carlo (MCMC) \citep[\eg,][]{neal2011mcmc} and variational inference \citep[see,][]{blei2017variational}. 
In our implementation, we used stochastic gradient langevin dynamics~\citep{welling2011bayesian,cheng2018convergence} to train an ensemble of models for posterior approximation.

\paragraph{Large State and Action Space} In general, we need to handle the case when the number of states and actions can be large, or even infinite. Notice that, when the state space is large, we can estimate the term $\int V_{h+1}(s^\prime)\mu(s^\prime) \dif s^\prime$ with regression based method using the samples from $f$~\citep{NIPS2007_da0d1111}. 
For the continuous action space, we can apply principled policy optimization methods \citep[\eg,][]{agarwal2020optimality} with an energy-based model~(EBM) parametrized policy~\citep{nachum2017bridging,dai2018sbeed}, treat the linear $Q^{\pi}(s, a)$ as the gradient and perform mirror descent and eventually obtain the optimal policy. However, this is at the cost of an additional sampling step from the EBM policy. In practice, we introduce a Gaussian policy and perform soft actor-critic \citep{haarnoja2018soft} policy update, which already provides good empirical performance. 
To sum up, for large state and action cases, we learn the critic in the learned representation space by regression, and obtain the Gaussian parametrized actor with SAC policy update step, in Line \ref{line:critic} and \ref{line:actor} in~\algtabref{alg:planning}, respectively. 

\paragraph{Infinite Horizon Case} Our algorithm can be provably extended to the infinite horizon case with specific termination condition for each episode \citep[\eg, see][]{jaksch2010near}. In practice, for the planning part we can solve the linear fixed-point equation with the feature $\phi(s, a)$ using the popular algorithms like Fitted $Q$-iteration (FQI)~\citep{NIPS2007_da0d1111} or dual embedding~\citep{dai2018sbeed}. that still guarantees to find the optimal policy. 

%% file: tex/analysis.tex
\section{Theoretical Guarantees}\label{sec:analysis}
In this section, we provide theoretical justification for~\algabb, showing that \algabb can identify informative representation and as a result, near-optimal policy in a sample-efficient way. 

We first define the notation of regret. Assume at episode $k$, the learner chooses the policy $\pi_k$ and observes a sequence $\{(s_h^k, a_h^k)\}_{h=0}^{H-1}$. We define the regret of the first $K$ episodes (and define $T:=KH$) as:
\begin{align}
    \mathrm{Regret}(K) := \sum_{k\in [K]} \left[V_0^*(s_0^k) - V_0^{\pi_k}(s_0^k)\right]
\end{align}
The regret measures the sample complexity of the representation learning in RL. We want to provide a regret upper bound that is sublinear in $T$. When $T$ increases, we collect more data that can help us build a much more accurate estimation on the representation, which should decrease the per-step regret and make the overall regret scale sublinear in $T$. As we consider the Thompson Sampling algorithm, we would like to study the expected regret $\mathbb{E}_{\mathbb{P}(f)} \left[\mathrm{Regret}(K)\right]$, which takes the prior $\mathbb{P}(f)$ into account.

\subsection{Assumptions}
Before we start, we first state the assumptions we use to derive our theoretical results.

We assume the reward is bounded, which is common in the literature \citep[\eg][]{azar2017minimax, jin2018q, jin2020provably}.
\begin{assumption}[Bounded Reward]
$r(s, a) \in [0, 1]$, $\forall (s, a) \in \mathcal{S}\times \mathcal{A}$.
\end{assumption}

In practice, we generally approximate $f^*$ with some complicated function approximators, so we focus on the setting where we want to find $f^*$ from a general function class $\mathcal{F}$ 
This is important for MuJoCo dynamics modeling, which have complicated transitions over angle, angular velocity and torque of the agent in the raw state.
We first state some necessary definitions and assumptions on $\mathcal{F}$.
\begin{definition}[$\ell_2$-norm of functions] 
Define
$\|f\|_2 := \max_{(s, a) \in \mathcal{S}\times\mathcal{A}} \|f(s, a)\|_2.$
Notice that it is not the commonly used $\ell_2$ norm for the function, but it suits our purpose well.
\end{definition}
\begin{assumption}[Bounded Output]
\label{assump:bounded_output}
We assume that $\|f\|_2 \leq C$, $\forall f\in\mathcal{F}$.
\end{assumption}
\begin{assumption}[Realizability]
\label{assump:realizability}
We assume the ground truth dynamic function $f^*\in\mathcal{F}$.
\end{assumption}
\vspace{-0.5em}
We then define the notion of covering number, which will be helpful in our algorithm derivation.

\begin{definition}[Covering Number \citep{wainwright2019high}]
An $\epsilon$-cover of $\mathcal{F}$ with respect to a metric $\rho$ is a set $\{f_i\}_{i\in [n]}\subseteq \mathcal{F}$, such that $\forall f\in \mathcal{F}$, there exists $i\in [n]$, $\rho(f, f_i) \leq \epsilon$. The $\epsilon$-covering number is the cardinality of the smallest $\epsilon$-cover, denoted as $\mathcal{N}(\mathcal{F}, \epsilon, \rho)$.
\end{definition}
\begin{assumption}[Bounded Covering Number]
\label{assump:bounded_covering} We assume that $\mathcal{N}(\mathcal{F}, \epsilon, \|\cdot\|_2) < \infty, \forall \epsilon > 0$.
\end{assumption}
\paragraph{Remark} Basically, Assumption \ref{assump:bounded_output} means the the transition dynamic never pushes the state far from the origin, which holds widely in practice. Assumption \ref{assump:realizability} guarantees that we can find the exact $f^*$ in $\mathcal{F}$, or we will always suffer from the error induced by model mismatch. Assumption \ref{assump:bounded_covering} ensures that we can estimate $f^*$ with small error when we have sufficient number of observations.

Besides the bounded covering number, we also need an additional assumption on bounded eluder dimension, which is defined in the following:
\begin{definition}[$\epsilon$-dependency \citep{osband2014model}]
A state-action pair $(s, a)\in\mathcal{S}\times\mathcal{A}$ is $\epsilon$-dependent on $\{(s_i, a_i)\}_{i\in [n]}\subseteq \mathcal{S}\times \mathcal{A}$ with respect to $\mathcal{F}$, if $\forall f, \tilde{f}\in\mathcal{F}$ satisfying $\sqrt{\sum_{i\in [n]} \|f(s_i, a_i) - \tilde{f}(s_i, a_i)\|_2^2}\leq \epsilon$ satisfies that $\|f(s, a) - \tilde{f}(s, a)\|_2 \leq \epsilon$. Furthermore, $(s, a)$ is said to be $\epsilon$-independent of $\{(s_i, a_i)\}_{i\in [n]}$ with respect to $\mathcal{F}$ if it is not $\epsilon$-dependent on $\{(s_i, a_i)\}_{i\in[n]}$.
\end{definition}
\begin{definition}[Eluder Dimension \citep{osband2014model}]
We define the eluder dimension $\mathrm{dim}_{E}(\mathcal{F}, \epsilon)$ as the length $d$ of the longest sequence of elements in $\mathcal{S}\times \mathcal{A}$, such that $\exists \epsilon^\prime\geq\epsilon$, every element is $\epsilon^\prime$-independent of its predecessors.
\end{definition}
\paragraph{Remark} Intuitively, eluder dimension illustrates the number of samples we need to make our prediction on unseen data accurate. If the eluder dimension is unbounded, then we cannot make any meaningful prediction on unseen data even with large amounts of collected samples. Hence, to make the learning possible, we need the following bounded eluder dimension assumption.
\begin{assumption}[Bounded Eluder Dimension]
\label{assump:bounded_eluder}
We assume $\mathrm{dim}_{E}(\mathcal{F}, \epsilon) < \infty, \forall \epsilon > 0$.
\end{assumption}

\subsection{Main Result}
\begin{theorem}[Regret Bound]
\label{thm:regret_bound}
Assume Assumption \ref{assump:bounded_output} to \ref{assump:bounded_eluder} holds. We have that
\begin{align*}
    & \mathbb{E}_{\mathbb{P}(f)}\left[\mathrm{Regret}(K)\right] \leq  \tilde{O}\bigg(\sqrt{H^2 T}\\
    & \cdot \sqrt{\log \mathcal{N}(\mathcal{F}, T^{-1/2}, \|\cdot\|_2)} \cdot \sqrt{\mathrm{dim}_{E}(\mathcal{F}, T^{-1/2})}\bigg).
\end{align*}
where $\tilde{O}$ represents the order up to logarithm factors.
\end{theorem}
For finite dimensional function class, $\log \mathcal{N}(\mathcal{F}, T^{-1/2}, \|\cdot\|_2)$ and $\mathrm{dim}_{E}(F, T^{-1/2}))$ should be scaled like $\mathrm{polylog}(T)$, hence our upper bound is sublinear in $T$. The proof is in Appendix \ref{sec:technical_proof}. Here we briefly sketch the proof idea.
\begin{proof}[Proof Sketch]
We first construct an equivalent UCB algorithm (see Appendix \ref{sec:ucb}) and bound $\mathrm{Regret}(K)$ for it. Then by the conclusion from \citet{russo2013eluder, russo2014learning, osband2014model}, we can directly translate the upper bound on $\mathrm{Regret}(K)$ from UCB algorithm to an upper bound on $\mathbb{E}_{\mathbb{P}(f)}\left[\mathrm{Regret}(K)\right]$ of TS algorithm. We emphasize that the UCB algorithm is solely designed for analysis purpose.

With the optimism, we know for episode $k$, $V_0^*(s_0^k) \leq \tilde{V}_{0, k}^{\pi_k}(s_0^k)$, where $\tilde{V}_{h, k}^{\pi_k}$ is the value function of policy $\pi_k$ under the model $\tilde{f}_k$ introduced in the UCB algorithm.
Hence, the regret at episode $k$ can be bounded by $\tilde{V}_{0, k}^{\pi_k}(s_0^k) - V_0^{\pi_k}(s_0^k)$, which is the value difference of the policy $\pi_k$ under the two models $\tilde{f}_k$ and $f^*$, that can be bounded by $\sqrt{\mathbb{E}\left[\sum_{h=0}^{H-1}\|f^*(s_h^k, a_h^k) - \tilde{f}_k(s_h^k, a_h^k)\|_2^2 \right]}$ (see Lemma \ref{lem:simulation} for the details), which means when the estimated model $\hat{f}$ is close to the real model $f^*$, the policy obtained by planning on $\hat{f}$ will only suffer from a small regret. With Cauchy-Schwartz inequality, we only need to bound
$\mathbb{E}\left[\sum_{k\in [K]}\sum_{h=0}^{H-1}\|f^*(s_h^k, a_h^k) - \tilde{f}_k(s_h^k, a_h^k)\|_2^2\right]$. This term can be handled via Lemma \ref{lem:width_sum_bound}. With some additional technical steps, we can obtain the upper bound on $\mathrm{Regret}(K)$ for the UCB algorithm, and hence the upper bound on $\mathbb{E}_{\mathbb{P}(f)}\left[\mathrm{Regret}(K)\right]$ for the TS algorithm.
\end{proof}
\paragraph{Kernelized Non-linear Regulator} Notice that, for the linear function class $\mathcal{F} = \{\theta^\top \varphi(s, a): \theta \in \mathbb{R}^{d_{\varphi} \times d}\}$ where $\varphi:\mathcal{S} \times \mathcal{A} \to \mathbb{R}^{d_{\varphi}}$ is a fixed and \emph{known} feature map of certain RKHS\footnote{Note that, the RKHS here is the Hilbert space that contains $f(s, a)$ with the feature from some fixed and known kernel, It is different from the RKHS we introduced in Section \ref{sec:algorithm}, that contains $Q(s, a)$ with the feature $k(f(s, a), \cdot)$ where $k$ is the Gaussian kernel.}, when the feature and the parameters are bounded, the logarithm covering number can be bounded by $\log \mathcal{N}(\mathcal{F}, \epsilon, \|\cdot\|_2)\lesssim d_{\varphi}\log (1/\epsilon)$, and the eluder dimension can be bounded by $\mathrm{dim}_{E}(\mathcal{F}, \epsilon) \lesssim d_{\varphi} \log (1/\epsilon)$ (see Appendix \ref{sec:linear_case} for the detail, notice that we provide a tighter bound of the eluder dimension compared with the one derived in \citet{osband2014model}). Hence, for linear function class, Theorem \ref{thm:regret_bound} can be translated into a regret upper bound of $\tilde{O}(H d_{\varphi} T^{1/2})$ for sufficiently large $T$, that matches the results of \citet{kakade2020information}\footnote{Note that $T$ in \citep{kakade2020information} is the number of episodes, and $V_{\max}$ in \citep{kakade2020information} can be viewed as $H^2$ when the per-step reward is bounded.}. Moreover, for the case of linear bandits when $H = 1$, our bound can be translated into a regret upper bound of $\tilde{O}(d_{\varphi} T^{1/2})$, that matches the lower bound \citep{dani2008stochastic} up to logarithmic terms.
\vspace{-3mm}
\paragraph{Compared with \citet{kakade2020information} and \citet{osband2014model}} Our results have some connections with the results from \citet{kakade2020information} and \citet{osband2014model}. However, in \citet{kakade2020information}, the authors only considers the case when $\mathcal{F}$ only contains linear functions w.r.t some known feature map, which constrains its application in practice. We instead, consider the general function approximation, which makes our algorithm applicable for more complicated models like deep neural networks. Meanwhile, the regret bound from \citet{osband2014model} depends on a global Lipschitz constant for the value function, which can be hard to quantify with either theoretical or empirical method. Instead, our regret bound gets rid of such dependency on the Lipschitz constant with the simulation lemma that carefully exploit the noise structure.

%% file: tex/experiments.tex
\section{Experiments}\label{sec:experiments}

\begin{table*}[t]
\caption{\footnotesize Performance of \algabb on various MuJoCo control tasks. All the results are averaged across 4 random seeds and a window size of 10K. Results marked with $^*$ is directly adopted from MBBL~\citep{wang2019benchmarking}. Our method achieves strong performance compared to pure empirical baselines (\eg, PETS). 
We also compare \algabb-REG which regularizes the critic using the model dynamics loss with several model-free RL method. \algabb-REG significantly improves the performance of the SoTA method SAC.
}
\scriptsize
\setlength\tabcolsep{3.5pt}
\label{tab:MuJoCo_results2}
\centering
\begin{tabular}{p{2cm}p{2cm}p{2cm}p{2.5cm}p{2cm}p{2cm}p{2cm}}
\toprule
& Swimmer & Reacher & MountainCar & Pendulum & I-Pendulum \\ 
\midrule  
ME-TRPO$^*$ & 30.1$\pm$9.7 & -13.4$\pm$5.2 & -42.5$\pm$26.6 & \textbf{177.3$\pm$1.9} & -126.2$\pm$86.6\\
PETS-RS$^*$  & 42.1$\pm$20.2 & -40.1$\pm$6.9 & -78.5$\pm$2.1 & 167.9$\pm$35.8 & -12.1$\pm$25.1\\
PETS-CEM$^*$  & 22.1$\pm$25.2 & -12.3$\pm$5.2 & -57.9$\pm$3.6 & 167.4$\pm$53.0 & -20.5$\pm$28.9\\
DeepSF & 25.5$\pm$13.5 & -16.8$\pm$3.6 & -17.0$\pm$23.4 & 168.6$\pm$5.1 & -0.2$\pm$0.3\\
{\bf \algabb} & \textbf{42.6$\pm$4.2} & \textbf{-7.2$\pm$1.1} & \textbf{50.3$\pm$1.1} & {169.5$\pm$0.6} & \textbf{0.0$\pm$0.0} \\
\midrule
PPO$^*$ & 38.0$\pm$1.5 & -17.2$\pm$0.9 & 27.1$\pm$13.1 & 163.4$\pm$8.0 & -40.8$\pm$21.0 \\
TRPO$^*$ & 37.9$\pm$2.0 & -10.1$\pm$0.6 & -37.2$\pm$16.4 & 166.7$\pm$7.3 & -27.6$\pm$15.8 \\
TD3$^*$ & 40.4$\pm$8.3 & -14.0$\pm$0.9 & -60.0$\pm$1.2 & 161.4$\pm$14.4 & -224.5$\pm$0.4 \\
SAC$^*$  & \textbf{41.2$\pm$4.6} & -6.4$\pm$0.5 & \textbf{52.6$\pm$0.6} & 168.2$\pm$9.5 & -0.2$\pm$0.1\\
{\bf \algabb-REG} & 40.0$\pm$3.8 & \textbf{-5.8$\pm$0.6} & 40.0$\pm$3.8 & \textbf{168.5$\pm$4.3} & \textbf{0.0$\pm$0.1}\\
\bottomrule 
\end{tabular}
\centering
\begin{tabular}{p{2cm}p{2cm}p{2cm}p{2.5cm}p{2cm}p{2cm}p{2cm}}
\toprule
& Ant-ET & Hopper-ET & S-Humanoid-ET & Humanoid-ET & Walker-ET \\ 
\midrule  
ME-TRPO$^*$ & 42.6$\pm$21.1 & 4.9$\pm$4.0 & 76.1$\pm$8.8 & 72.9$\pm$8.9 & -9.5$\pm$4.6\\
PETS-RS$^*$ & 130.0$\pm$148.1 &  205.8$\pm$36.5 & 320.9$\pm$182.2 & 106.9$\pm$106.9 & -0.8$\pm$3.2 \\
PETS-CEM$^*$ & 81.6$\pm$145.8 & 129.3$\pm$36.0 & 355.1$\pm$157.1 & 110.8$\pm$91.0 & -2.5$\pm$6.8 \\
DeepSF & 768.1$\pm$44.1  & 548.9$\pm$253.3 & 533.8$\pm$154.9 & 168.6$\pm$5.1 & 165.6$\pm$127.9\\
{\bf \algabb} & \textbf{806.2$\pm$60.2} & \textbf{732.2$\pm$263.9} & \textbf{986.4$\pm$154.7} & \textbf{886.9$\pm$95.2} & \textbf{501.6$\pm$204.0}  \\
\midrule
PPO$^*$ & 80.1$\pm$17.3  & 758.0$\pm$62.0 & 454.3$\pm$36.7 & 451.4$\pm$39.1 & 306.1$\pm$17.2\\
TRPO$^*$ & 116.8$\pm$47.3  & 237.4$\pm$33.5 & 281.3$\pm$10.9 & 289.8$\pm$5.2 & 229.5$\pm$27.1\\
TD3$^*$ & 259.7$\pm$1.0  & 1057.1$\pm$29.5 & 1070.0$\pm$168.3 & 147.7$\pm$0.7 & \textbf{3299.7$\pm$1951.5}\\
SAC$^*$ & {\bf 2012.7$\pm$571.3}  & 1815.5$\pm$655.1 & 834.6$\pm$313.1 & 1794.4$\pm$458.3 & 2216.4$\pm$678.7\\
{\bf \algabb-REG} & \textbf{2073.1$\pm$119.7} & \textbf{2510.3$\pm$550.8} & \textbf{2710.3$\pm$277.5} & \textbf{3747.8$\pm$1078.1} & 2170.3$\pm$810.9 \\
\bottomrule 
\end{tabular}
\end{table*}
In this section, we study the empirical performance of \algabb in the OpenAI MuJoCo control suite~\citep{1606.01540}.
We use the environments from MBBL~\citep{wang2019benchmarking}, which varies slightly from the original environments in terms of modifying the reward function so its gradient w.r.t. the states exists and introducing early termination (ET). Note that the set of environments contains various control and manipulation tasks, which are commonly used for benchmarking both model-free and model-based RL algorithms~\citep[\eg,][]{kakade2020information, haarnoja2018soft}. As aforementioned, for practical implementation, our critic network consists of a representation network $\phi(\cdot)$ and a linear layer on the top. We follow the same procedure of Algorithm ~\ref{alg:TS}. Specifically, (1) for finding the optimal policy, we run an actor-critic algorithm (SAC); (2) we fix the representation network of the critic function $\phi(\cdot)$ and only update the linear layer on the top. We provide the full set of experiments in Appendix~\ref{appendix:full_exp} and the hyperparameter we use in Appendix~\ref{appendix:hyperparam}.~\footnote{Our code is available at \href{https://sites.google.com/view/spede}{https://sites.google.com/view/spede}.}
\paragraph{Baselines} We compare our method with various model-based RL baselines: PETS~\citep{chua2018deep} with random shooting (RS) optimizer, PETS with cross entropy method (CEM) optimizer and ME with TRPO policy optimizer~\citep{kurutach2018model}. Note that these are strong empirical baselines with many hand-tuned hyperparameters and engineering features (\eg, ensemble of models). It is usually hard for any theoretically guaranteed model-based RL algorithm to match or surpass their performance~\citep{kakade2020information}. Another natural baseline is the successor feature~\citep{dayan1993improving}, which is one of the representative spectral features. We compare with the deep successor feature (DeepSF)~\citep{kulkarni2016deep}, and for a fair comparison, we only swap the representation objective of \algabb with DeepSF and keep the other parts of the algorithm exactly the same.
\paragraph{\algabb: Performance with the Learned Representation} Following Algorithm~\ref{alg:TS}, we are interested in how \algabb performs when we conduct planning on top of the representation induced by the dynamics model in each episode. 
As most of the rigorously-justified representation learning algorithms are computationally intractable/inefficient, to demonstrate the effectiveness of representation used in \algabb, we compare \algabb with the deep successor features, which is one representative empirical representation learning algorithm. Moreover, as our method learning representation via fitting transition dynamics, to demonstrate the superiority of representation in planning, we compare our methods with the state-of-the-art model-based RL algorithms.
We summarize the results of our method in Table~\ref{tab:MuJoCo_results2}. We see that our method achieves impressive performance comparing to model-based RL methods. Even in some hard environments that baselines fail to reach positive reward (\eg, MountainCar, Walker-ET), \algabb manage to achieve a reward of 52.6 and 501.6 respectively. We also evaluate our representation by comparing \algabb to the usage of deep successor feature (DeepSF). Results show that on hard tasks like Humanoid and Walker, \algabb manages to achieve 452.6 and 336.0 higher reward respectively. 
\paragraph{\algabb-REG: Policy Optimization with \algabb Representation Regularizer}
In order to evaluate whether our assumption on linear MDP is valid in empirical settings and study whether such assumption can help improve the performance, 
we add our model dynamics representation objective as a regularizer in addition to the original SAC algorithm for learning the $Q$-function. Specifically, the algorithm \algabb-REG consists of vanilla SAC objective with an additional loss putting constraints on the representation learned by the critic function, due to the intuition that the representation should satisfy the equivalent dynamics. We compare its performance with the vanilla SAC algorithm to show the benefits of dynamic representation. Results in Table~\ref{tab:MuJoCo_results2} show that adding such constraint significantly improve the performance of SAC: on hard tasks like Hopper-ET, S-Humanoid-ET and Humanoid-ET, \algabb-REG improves the performance of SAC by 694.8, 1875.7 and 2000.4. 
\begin{figure*}[t]
    \centering
    \includegraphics[width=0.8\textwidth]{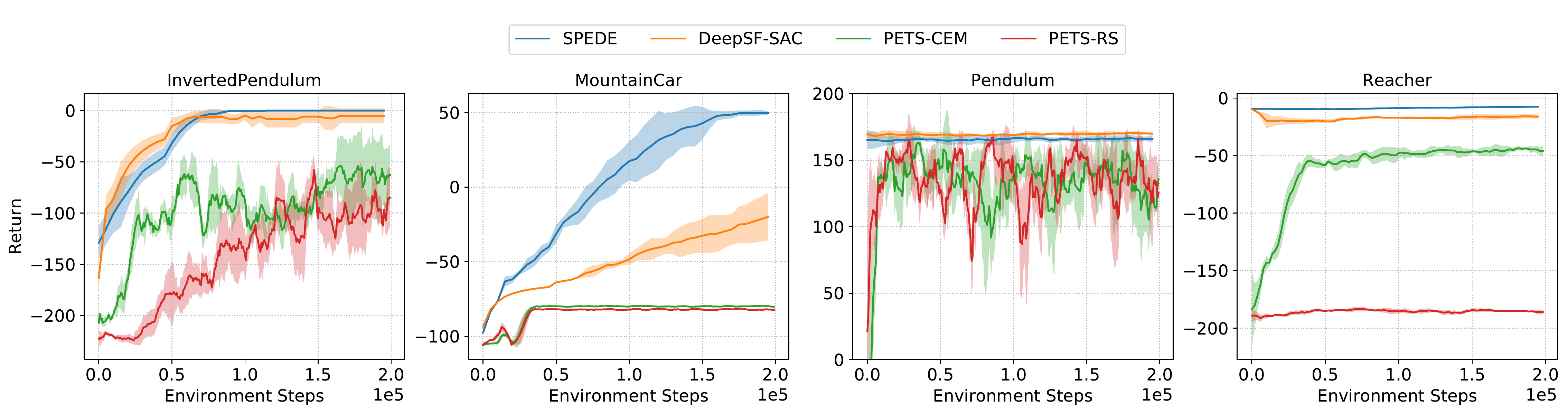}
    \caption{\footnotesize \textbf{Experiments on MuJoCo:} We show curves of the return versus the training steps for \algabb and model-based RL baselines. Results show that in these tasks, our method enjoys better sample efficiency even compared to SoTA empirical model-based RL baselines.}
    \label{fig:MuJoCo}
\end{figure*}

\paragraph{Ablations} We conduct ablations on: (1) What is the effect of the momentum parameter. (2) How does the number of random features affect the performance. 
Detailed results can be found at ~\ref{appendix:ablations}.

\paragraph{Performance Curves} To better understand how the sample complexity of our algorithm comparing to the prior model-based RL baselines, we plot the return versus environment steps in Figure~\ref{fig:MuJoCo}. We see that comparing to prior model-based baselines, \algabb enjoys great sample efficiency in these tasks. We want to emphasize that from MBBL~\citep{wang2019benchmarking}, model-based methods already show significantly better sample efficiency compared to model-free methods (\eg PPO/TRPO). We provide additional results in Appendix~\ref{appendix:full_exp}.

\paragraph{Discussion of the Results} We observe that in the environments with relatively simple dynamics (top row of Table~\ref{tab:MuJoCo_results2}), \algabb achieves the SoTA among all the model-based and model-free RL algorithms.
When the model dynamics of the environment become harder (bottom row of Table~\ref{tab:MuJoCo_results2}), the difference of the performance between the two approaches begin to enlarge. 
Interestingly, our \algabb achieves strong results comparing to model-based approaches, while the joint learning \algabb-REG outperforms model-free algorithm by a huge margin. The performance promotion of \algabb indicates the importance on learning a good representation based on model dynamics and again shows the effectiveness of our approach in both settings. The performance gap might be caused by random feature approximation. To mitigate such approximation error, we also tried using MLP upon the learned representation, instead of linear form, which leads to better performances. Please refer to~\appref{appendix:ablations} for details. 

In fact, the differences in the SoTA usage of \algabb in easy environments and difficult environments also reveals the important direction for our future work. The current rigorous representation learning methods, \eg,~\citet{du2019provably, misra2020kinematic, agarwal2020flambe} and the proposed~\algabb, all rely on some model assumption. When the assumptions are satisfied, \eg, Pendulum, Reacher, and others, our theoretically derived \algabb variant works extremely well, even better than current SoTA. However, when the assumption is not fully satisfied, although the decoupled \algabb achieves best performance among existing model-based RL and representation learning under fair comparison, the joint learned variant of \algabb is more robust and promotes the current SoTA with significant margin. An interesting question is whether we can rigorously justify the regularized \algabb , which we leave as our future work.

%% file: tex/conclusion.tex
\section{Conclusion}
We introduce \algabb, which, to the best of our knowledge, is the first provable and efficient representation learning algorithm for RL, by exploiting the benefits from noise. We provide thorough theoretical analysis and strong empirical results, comparing to both model-free and model based RL, that demonstrates the effectiveness of our algorithm. 

%% file: tex/acknowledgement.tex
\section*{Acknowledgement}
Cs. Sz. greatly acknowledges funding from NSERC, AMII and the Canada CIFAR AI Chair program. This project occurred under the Google-BAIR Commons at UC Berkeley.

%% file: tex/appendix_background.tex
\section{Backgrounds on Reproducing Kernel Hilbert Space}
\label{sec:background}
We briefly introduce the basic concepts of the Reproducing Kernel Hilbert Space, which is helpful on understanding our paper. To start with, we first define the inner product.
\begin{definition}[Inner Product]
A function $\langle \cdot, \cdot \rangle_{\mathcal{H}} : \mathcal{H}\times \mathcal{H} \to \mathbb{R}$ is said to be an inner product on $\mathcal{H}$ if it satisfies the following conditions:
\begin{enumerate}
    \item Positive Definiteness: $\forall u \in \mathcal{H}$, $\langle u, u \rangle\geq 0$, and $\langle u, u\rangle = 0 \Longleftrightarrow u = 0$.
    \item Symmetry: $\forall u, v \in \mathcal{H}$, $\langle u, v \rangle\in \langle v, u\rangle$.
    \item Bilinearity: $\forall \alpha, \beta \in \mathbb{R}, u, v, w\in\mathcal{H}$, $\langle \alpha u + \beta v, w \rangle = \alpha\langle u, w\rangle + \beta \langle v, w\rangle$.
\end{enumerate}
Additionally, we can define a norm with the inner product: $\|u\| = \sqrt{\langle u, u\rangle}$.
\end{definition}
A Hilbert space is a space equipped with an inner product and satisfies an additional technical condition of completeness. The finite-dimension vector space with the canonical inner product is an example of the Hilbert space. We remark that $\mathcal{H}$ can also be a function space, for example, the space contains all square integrable functions (i.e. $\int_{\mathbb{R}} f(x)^2 \dif x < \infty$, generally denoted as $L_2$) is also a Hilbert space with inner product $\langle f, g\rangle = \int_{\mathbb{R}} f(x) g(x) \dif x$. 

We then define the kernel, and introduce the notion of positive-definite kernel \citep{alvarez2012kernels}.
\begin{definition}[(Positive-Definite) Kernel]
A function $k:\mathcal{X} \times \mathcal{X} \to \mathbb{R}$ is said to be a kernel on non-empty set $\mathcal{X}$ if there exists a Hilbert space $\mathcal{H}$ and a feature map $\phi:\mathcal{X}\to\mathcal{H}$ such that $\forall x, x^\prime\in\mathcal{X}$, we have
\begin{align*}
    k(x, x^\prime) = \langle \phi(x), \phi(x^\prime)\rangle_{\mathcal{H}}.
\end{align*}
Moreover, the kernel is said to be positive definite if $\forall n\geq 1$, $\forall \{a_i\}_{i\in [n]} \subset \mathbb{R}$ and mutually distinct set $\{x_i\}_{i\in [n]} \subset \mathcal{X}$, we have that
\begin{align*}
    \sum_{i\in[n]}\sum_{j\in[n]} a_i a_j k(x_i, x_j) > 0. 
\end{align*}
\end{definition}
Some well-known kernels include:
\begin{itemize}
    \item Linear Kernel: $k(x, y) = \langle x, y\rangle$, with the canonical feature map $\phi(x) = x$.
    \item Polynomial Kernel: $k(x, y) = (\langle x, y \rangle + c)^m$, where $m\in\mathbb{N}^+$ and $c\in\mathbb{R}^+$.
    \item Gaussian (a.k.a radial basis function, RBF) Kernel: $k(x, y) = \exp\left(\frac{\|x-y\|_2^2}{2\sigma^2}\right)$. It's known that such kernel is positive definite.
\end{itemize}
Now we can define the Reproducing Kernel Hilbert space (RKHS) \citep{aronszajn1950theory}.
\begin{definition}[Reproducing Kernel Hilbert Space (RKHS)]
The Hilbert space $\mathcal{H}$ of $\mathbb{R}$-valued function defined on a non-emptry set $\mathcal{X}$ is said to be a reproducing kernel Hilbert space (RKHS) is there is a kernel $k:\mathcal{X} \times \mathcal{X} \to \mathbb{R}$, such that
\begin{enumerate}
    \item $\forall x\in\mathcal{X}$, $k(x, \cdot) \in \mathcal{H}$.
    \item $\forall x\in\mathcal{X}, f\in\mathcal{H}$, $\langle f, k(x, \cdot)\rangle_{\mathcal{H}} = f(x)$ (a.k.a the reproducing property), which also implies that $\langle k(x, \cdot), k(y, \cdot)\rangle = k(x, y)$.
\end{enumerate}
Here $k$ is called a reproducing kernel of $\mathcal{H}$.
\end{definition}
We provide an intuitive interpretation on the definition of RKHS when $\mathcal{H}$ is the space of linear function. Consider $\mathcal{X} = \mathbb{R}^d$ and $k(x, y) = \langle x, y \rangle$. With the definition of the kernel $k$, we can see that $k(x, \cdot) : \mathcal{X} \to \mathbb{R}$ is a linear function, and thus lies in $\mathcal{H}$. Meanwhile, $\forall f\in\mathcal{H}$, there exists $\theta_f$ such that $f(x) = \theta_f^\top x$. We define the inner product on $\mathcal{H}$ via $\langle f, g\rangle_{\mathcal{H}} = \langle \theta_f, \theta_g\rangle$, and thus $\langle f(, k(x, \cdot))\rangle_{\mathcal{H}} = \theta^\top x = f(x)$, which demonstrates the reproducing property, and shows that the space of linear function on any finite-dimensional vector space is an RKHS with linear kernel as the corresponding reproducing kernel.

We state the following theorems without the proof.
\begin{theorem}[Moore-Aronszajn \citep{aronszajn1950theory}]
Every positive definite kernel $k$ is associated with a unique RKHS $\mathcal{H}$.
\end{theorem}
Notice that, Moore-Aronszajn theorem guarantees that all of the positive kernel can be represented as the inner product in certain Hilbert space, hence we can have a linear representation of the Gaussian distribution induced by the reproducing property of Gaussian kernel, as we illustrated in the main text.
\begin{theorem}[Bochner \citep{rudin2017fourier}]
A continuous, shift-invariant kernel (i.e. $k(x, y) = k(x-y)$) is positive definite if and only if $k(x-y)$ is the Fourier transform of a non-negative measure $\omega$, i.e.
\begin{align*}
    k(x-y) = \int_{\mathbb{R}^d} \exp(i\omega^\top (x-y)) \dif \mathbb{P}(\omega) = \int_{\mathbb{R}^d \times [0, 2\pi]} 2\cos(\omega^\top x + b) \cos(\omega^\top y + b) \dif (\mathbb{P}(\omega) \times \mathbb{P}(b)),
\end{align*}
where $\mathbb{P}(b)$ is a uniform distribution on $[0, 2\pi]$.
\end{theorem}
Bochner's theorem shows that any continuous positive definite shift-invariant kernel (e.g. Gaussian kernel, Laplacian kernel) can be represented as the inner product of random Fourier feature, which provides an additional way to provide a representation for certain distribution \citep[see][]{rahimi2007random, dai2014scalable}.

%% file: tex/appendix_ucb.tex
\section{An Equivalent Upper Confidence Bound Algorithm}
\label{sec:ucb}
In this section, we provide a generic Upper Confidence Bound (UCB) algorithm with the OFU principle, and show the connections and differences between the UCB algorithm and the TS algorithm. The prototype for our UCB algorithm is illustrated in Algorithm \ref{alg:UCB}.
\begin{algorithm}
\caption{Upper Confidence Bound (UCB) Algorithm} 
\label{alg:UCB}
\begin{algorithmic}[1]
\Require Number of Episodes $K$, Failure Probability $\delta\in(0, 1)$, Reward Function $r(s, a)$.
\State Initialize the history set $\mathcal{H}_0 = \emptyset$.
\For{episodes $k=1, 2, \cdots$}
\State \label{line:optimistic_planning} {\color{blue} Compute $\pi_{k}$ via \Comment{Optimistic Planning.}
\begin{align*}
    (\pi_{k}, \tilde{f}_k) = \mathop{\arg\max}_{\pi\in\Pi, \tilde{f}\in\mathcal{F}_k} \tilde{V}_{0}^\pi(s_0).
\end{align*}
where $\mathcal{F}_k$ is defined in \eqref{eq:confidence_set}.}
\For{steps $h=0, 1, \cdots, H-1$}\Comment{Execute $\pi_k$.}
\State Execute $a_h^k \sim \pi_k^h(s_h^k)$.
\State Observe $s_{h+1}$.
\EndFor
\State Set $\mathcal{H}_k = \mathcal{H}_{k-1} \cup \{(s_h^k, a_h^k, s_{h+1}^k)\}_{h=0}^{H-1}$.\Comment{Update the History.}
\EndFor
\end{algorithmic}
\end{algorithm}

Notice that, the only difference between UCB algorithm and TS algorithm is the mechanism of finding $f$ we use to plan for each episode (highlighted in blue). For UCB algorithm, we perform an optimistic planning, which finds the $\tilde{f}_k$ that potentially has the largest cumulative reward. However, such constrained optimization problem is NP-hard even for the simplest linear bandits \citep{dani2008stochastic}. Instead, for TS algorithm, we only sample the $f_k$ from the posterior distribution, which gets rid of the complicated constraint optimization. We are interested in the UCB algorithm, as the worst case regret bound of the UCB algorithm can be directly translated to the expected regret bound of the TS algorithm without the need of explicit manipulation of the prior and the posterior\citep{russo2013eluder, russo2014learning, osband2014model}.

\paragraph{Confidence Set Construction} Perhaps the most important part in OFU-style algorithm is the construction of confidence set $\mathcal{F}_k$. To enable sample-efficient learning, the confidence set should
\begin{enumerate}
    \item contain $f^*$ with high probability, so that we can identify $f^*$ eventually;
    \item shrink as fast as possible, so that we can identify $f^*$ efficiently.
\end{enumerate}
In the tabular setting, $\mathcal{F}_k$ is constructed via the concentration of sub-Gaussian/sub-Gamma random variable \citep[e.g.][]{azar2017minimax}, and in the linear MDP setting, $\mathcal{F}_k$ is constructed via the concentration on the linear parameters. As we don't assume any specific structures, we instead constructed $\mathcal{F}_k$ via the concentration on the $\ell_2$ error, following the idea of \citep{russo2013eluder, osband2014model}. Specifically, consider the least-square estimates defined by
\begin{align}
    \hat{f}_K = \mathop{\arg\min}_{f\in \mathcal{F}}L_{2, K}(f) := \sum_{k\in [K]}\sum_{h=0}^{H-1} \|f(s_h^k, a_h^k) - s_{h+1}^k\|_2^2.
\end{align}
As $s_{h+1}^k = f^*(s_h^k, a_h^k) + \epsilon_h^k$ where $\epsilon_h^k$ is the Gaussian noise added to the step $h$ at the $k$-th episode, we know $\hat{f}_{K}$ will not deviate from $f^*$ a lot. Meanwhile, as $K$ increases, the estimation $\hat{f}_K$ should become closer to $f^*$. Specifically, define the empirical $2$-norm $\|\cdot\|_{2, E_t}$ as
\begin{align*}
    \|g\|_{2, E_K}^2 := \sum_{k\in [K]} \sum_{h=0}^{H-1} \|g(s_h^k, a_h^k)\|_2^2.
\end{align*}
We can construct the confidence set based on the following lemma:
\begin{lemma}[Confidence Set Construction \citep{russo2013eluder, osband2014model}]\label{lem:confidence_set}
Define
\begin{align}
\label{eq:confidence_set}
    \mathcal{F}_K = \left\{f\in\mathcal{F}:\|f - \hat{f}_K\|_{2, E_K}\leq \sqrt{\beta_K^*(\mathcal{F}, \delta, \alpha)}\right\},
\end{align}
then
\begin{align}
    \mathbb{P}_{f^*}\left(f^*\in \bigcap_{k=1}^{\infty}\mathcal{F}_k\right) \geq 1-2\delta,
\label{eq:optimism}
\end{align}
where
\begin{align}
    \beta_K^*(\mathcal{F}, \delta, \alpha) = 8\sigma^2 \log(\mathcal{N}(\mathcal{F}, \alpha, \|\cdot\|_{2})/\delta) + 2H\alpha(12C + \sqrt{8d\sigma^2\log(4K^2H/\delta)}).
\end{align}
\end{lemma}
The proof can be found in Appendix \ref{sec:proof_optimism}. Notice that, the empirical $2$-norm $\|f - \hat{f}_K\|_{2, E_K}$ scales linearly with $K$, and $\beta_K^*(\mathcal{F}, \delta, \alpha)$ only scales as $\log K$, so the confidence set shrinks. Meanwhile, Equation \ref{eq:optimism} guarantees that $f^* \in \mathcal{F}_k$, $\forall k$ with high probability. Hence, it satisfies our requirement for the confidence set.

{\color{black}
\paragraph{Regret Upper Bound} We have the following upper bound of the regret for the UCB algorithm:
\begin{theorem}[Regret Bound]
\label{thm:regret_bound_UCB}
Assume Assumption \ref{assump:bounded_output} to \ref{assump:bounded_eluder} holds. We have that
\begin{align*}
\textstyle
    \mathrm{Regret}(K) \leq \tilde{O}(\sqrt{H^2 T\cdot \log \mathcal{N}(\mathcal{F}, T^{-1/2}, \|\cdot\|_2) \cdot \mathrm{dim}_{E}(\mathcal{F}, T^{-1/2})}).
\end{align*}
where $\tilde{O}$ represents the order up to logarithm factors.
\end{theorem}
}

%% file: tex/appendix_proof.tex
\section{Technical Proof}
\label{sec:technical_proof}
\subsection{Proof for Lemma \ref{lem:confidence_set}} 
\label{sec:proof_optimism}
\begin{proof} We first show the following concentration on the $\ell_2$ error:
\begin{lemma}[Concentration of $\ell_2$ error \citep{russo2013eluder, osband2014model, wang2020reinforcement}]
$\forall \delta > 0, f:\mathcal{S}\times \mathcal{A} \to \mathbb{R}$, we have
\begin{align*}
    \mathbb{P}_{f^*}\left(L_{2, K}(f) \geq L_{2, K}(f^*) + \frac{1}{2}\|f - f^*\|^2_{2, E_K}- 4\sigma^2 \log (1/\delta), \quad\forall K\in \mathbb{N}\right) \geq 1-\delta
\end{align*}
\end{lemma}
\begin{proof}
Define the filtration $\mathcal{H}_{k, h} = \{(s_h^i, a_h^i)\}_{i\in [k-1], h = 0, \cdots, H-1}\cup \{(s_i^k, a_i^k)\}_{h=0}^{h-1}$, and the random variable $Z_{k,h}$ adapted to the filtration $\mathcal{H}_{k, h}$ via:
\begin{align*}
    Z_{k, h} = & \|f^*(s_h^k, a_h^k) - s_{h+1}^k\|^2_2 - \|f(s_h^k, a_h^k) - s_{h+1}^k\|^2_2\\
    = & \|f^*(s_h^k, a_h^k) - s_{h+1}^k\|^2_2 - \|f(s_h^k, a_h^k) - f^*(s_h^k, a_h^k) + f^*(s_h^k, a_h^k) - s_{h+1}^k \|^2_2\\
    = & - \|f(s_h^k, a_h^k) - f^*(s_h^k, a_h^k)\|^2_2 + 2\langle f(s_h^k, a_h^k) - f^*(s_h^k, a_h^k), \epsilon_h^k\rangle,
\end{align*}
where $\epsilon_h^k = s_{h+1}^k - f^*(s_h^k, a_h^k)$. Thus, $\mathbb{E}(Z_k^h|\mathcal{H}_{k, h}) = - \|f(s_h^k, a_h^k) - f^*(s_h^k, a_h^k)\|^2_2$, and $Z_k^h + \|f(s_h^k, a_h^k) - f^*(s_h^k, a_h^k)\|^2_2$ is a martingale w.r.t $\mathcal{H}_{k, h}$. Notice that we assume $\epsilon$ is an isotropic Gaussian noise with variance $\sigma^2$ on each of the dimension, thus the conditional moment generating function of $Z_k^h + \|f(s_h^k, a_h^k) - f^*(s_h^k, a_h^k)\|^2_2$ satisfies:
\begin{align*}
    M_{k, h}(\lambda) = & \log \mathbb{E}[\exp(\lambda(Z_k^h + \|f(s_h^k, a_h^k) - f^*(s_h^k, a_h^k)\|^2_2))|\mathcal{H}_{k, h}] \\
    = & \log \mathbb{E}[\exp(\langle 2\lambda f(s_h^k, a_h^k) - f^*(s_h^k, a_h^k), \epsilon_h^k \rangle)|\mathcal{H}_{k, h}]\\
    \leq & 2\sigma^2\lambda^2\|f(s_h^k) - f^*(s_h^k, a_h^k)\|_2^2.
\end{align*}
Applying Lemma 4 in \citep{russo2013eluder}, we have that, $\forall x, \lambda \geq 0$,
\begin{align*}
    \mathbb{P}_{f^*}\left(\sum_{k\in [K]}\sum_{h=0}^{H-1}\lambda Z_{k, h} \leq x - \lambda(1-2\lambda  \sigma^2)\sum_{k\in [K]}\sum_{h=0}^{H-1}\|f(s_h^k, a_h^k) - f^*(s_h^k, a_h^k)\|_2^2,\quad \forall k\in \mathbb{N} \right) \leq 1-\exp(-x).
\end{align*}
Take $\lambda = \frac{1}{4\sigma^2}, x = \log 1/\delta$, and notice that $\sum_{k\in [K]}\sum_{h=0}^{H-1}Z_{k, h} = L_{2, K}(f^*) - L_{2, K}(f)$, we have the desired result.
\end{proof}

We construct an $\alpha$-cover $\mathcal{F}_{\alpha}$ in $\mathcal{F}$ with respect to $\|\cdot\|_2$. With a standard union bound, we know that condition on $f^*$, with probability at least $1-\delta$, we have that
\begin{align*}
    L_{2, K}(f^{\alpha}) - L_{2, K}(f^*) \geq \frac{1}{2}\|f^\alpha - f^*\|_{2, E_{K}}^2 - 4\sigma^2 \log (|F^{\alpha}|/\delta), \quad \forall K\in \mathbb{N}, f^{\alpha} \in \mathcal{F}^{\alpha}.
\end{align*}
Thus, we have that
\begin{align*}
    L_{2, K}(f) - L_{2, K}(f^*) \geq & \frac{1}{2}\|f - f^*\|_{2, E_{K}}^2 - 4\sigma^2 \log (|F^{\alpha}|/\delta)\\
    & + \underbrace{\min_{f^\alpha \in \mathcal{F}^{\alpha}}\left\{\frac{1}{2}\|f^{\alpha} - f^*\|_{2, E_K}^2 - \frac{1}{2}\|f - f^*\|_{2, E_K}^2 + L_{2, K}(f) - L_{2, K}(f^\alpha)\right\}}_\text{Discretization Error}.
\end{align*}
We then deal with the discretization error. Assume $\alpha \leq 2C$ (or otherwise we only have a trivial cover) and $\|f^{\alpha}(s, a) - f(s, a)\|_2 \leq \alpha$, we have that
\begin{align*}
    & \|f^{\alpha}(s, a) - f^*(s, a)\|_2^2 - \|f(s, a) - f^*(s, a)\|_2^2\\
    = & \|f^\alpha(s, a)\|_2^2-\|f(s, a)\|_2^2 + 2\langle f^*(s, a), f(s, a) - f^{\alpha}(s, a)\rangle\\
    \leq &\max_{\|y\|_2\leq \alpha}\{\|f(s, a) + y\|_2^2 - \|f(s, a)\|_2^2\} + 2C\alpha\\
    = & \max_{\|y\|_2\leq \alpha} \{2\langle f(s, a), y\rangle + \|y\|_2^2\} + 2C\alpha\\
    \leq & 4C\alpha + \alpha^2 \leq 6C\alpha,
\end{align*}
where the inequality is by Cauchy-Schwartz inequality and $\alpha \leq 2C$. Meanwhile,
\begin{align*}
    & \|s^\prime - f(s, a)\|_2^2 - \|s^\prime - f^{\alpha}(s, a)\|_2^2 \\
    = & 2\langle s^\prime, f^{\alpha}(s, a) - f(s, a)\rangle + \|f(s, a)\|_2^2 - \|f^{\alpha}(s, a)\|_2^2\\
    \leq & 2\langle \epsilon, f^{\alpha}(s, a) - f(s, a)\rangle + 2\langle f^*(s, a), f^{\alpha}(s, a) - f(s, a)\rangle + 2C\alpha + \alpha^2\\
    \leq & 2\|\epsilon\|_2 \alpha + 6C\alpha.
\end{align*}
We now consider the concentration property of $\|\epsilon\|_2$. Here we simply follow \citep{jin2019short} and notice that $\epsilon$ is $\sqrt{d}\sigma$-norm-sub-Gaussian, we have that
\begin{align*}
    \mathbb{P}(\|\epsilon\|_2 > \sqrt{2d\sigma^2 \log (2/\delta)}) \leq \delta.
\end{align*}
By a union bound, we have that
\begin{align*}
    \mathbb{P}(\exists k, \|\epsilon\|_2 > \sqrt{2d\sigma^2 \log (4k^2H/\delta)}) \leq \frac{\delta}{2} \sum_{k=1}^{\infty}\sum_{h=0}^{H-1} \frac{1}{k^2 H} \leq \delta.
\end{align*}
Sum all these up, we can see with probability $1-\delta$, $\forall K\in\mathbb{N}$, the discretization error is upper bounded by:
\begin{align*}
    H\alpha(12C + \sqrt{8d\sigma^2 \log (4K^2 H/\delta)}).
\end{align*}
As we consider the least square estimate $\hat{f}_K$, we have that $L_{2, K}(\hat{f}_K) - L_{2, K}(f^*) \leq 0$. Substitute back, we have the desired results.
\end{proof}
\subsection{Simulation Lemma}
\begin{lemma}[Simulation Lemma (adapted from Lemma 3.9 in \citep{kakade2020information})]\label{lem:simulation}
Given $\hat{f}$, $\forall s\in\mathcal{S}$, the value function $\hat{V}^\pi$ and $V^\pi$ corresponding to the model $\hat{f}$ and $f^*$ satisfies
\begin{align*}
    \hat{V}_0^\pi(s) - V_0^\pi(s) \leq H^{3/2} \sqrt{\mathbb{E}\left[\sum_{h=0}^{H-1}\min\left\{\frac{2\|f^*(s_h, a_h) - \hat{f}(s_h, a_h) \|_2^2}{\sigma^2}, 1\right\}\right]}.
\end{align*}
\end{lemma}
\label{sec:proof_simulation}
\begin{proof}

We first show the following difference lemma:
\begin{lemma}[Difference Lemma]
\label{lem:difference}
Assume the trajectory $\{(s_h, a_h)\}_{h=0}^{H-1}$ is generated via policy $\pi$ and ground truth $f^*$, define
\begin{align*}
    V_h = \sum_{\tau=h}^{H-1}r(s_\tau, a_\tau)
\end{align*}
then $\forall \tau \in \{1, \cdots, H-1\}$, we have:
\begin{align*}
    \hat{V}_0^\pi(s_0) - V_0  = & \mathbb{E}_{s_\tau^\prime\sim \mathcal{N}(\hat{f}(s_{\tau-1}, a_{\tau-1}), \sigma^2 I)} \left[\hat{V}_{\tau}^\pi(s_\tau^\prime)\right] - V_{\tau} \\
    & + \sum_{h=1}^{\tau-1}\left[\mathbb{E}_{s_h^\prime\sim \mathcal{N}(f(s_{h-1}, a_{h-1}), \sigma^2 I)} \left[\hat{V}_h^\pi(s_h^\prime)\right] - \hat{V}_h^\pi(s_h) \right].
\end{align*}
\end{lemma}
\begin{proof}
When $\tau = 1$, we can obtain the result with $a_0 = \pi(s_0)$ and
\begin{align*}
    \hat{V}_0^\pi(s_0) = r(s_0, \pi(s_0)) + \mathbb{E}_{s_1^\prime\sim \mathcal{N}(f(s_0, a_0), \sigma^2 I)} \hat{V}_1^\pi(s_1^\prime).
\end{align*}
We only need to show the case when $\tau = 2$, and the case when $\tau > 2$ can be derived via recursion. Notice that
\begin{align*}
    \hat{V}_0^\pi(s_0) - V_0  = & \mathbb{E}_{s_1^\prime\sim \mathcal{N}(f(s_0, a_0), \sigma^2 I)} \left[\hat{V}_1^\pi(s_1^\prime)\right] - V_1\\
    = & \hat{V}_1^\pi(s_1) - V_1  +  \mathbb{E}_{s_1^\prime\sim \mathcal{N}(f(s_0, a_0), \sigma^2 I)} \left[\hat{V}_1^\pi(s_1^\prime)\right] - \hat{V}_1^\pi(s_1)\\
    = & \mathbb{E}_{s_2^\prime\sim \mathcal{N}(f(s_1, a_1), \sigma^2 I)}\left[\hat{V}_2^\pi(s_2^\prime)\right] - V_2 + \mathbb{E}_{s_1^\prime\sim \mathcal{N}(f(s_0, a_0), \sigma^2 I)} \left[\hat{V}_1^\pi(s_1^\prime)\right] - \hat{V}_1^\pi(s_1),
\end{align*}
where the last equality is due to the fact that $a_1 = \pi(s_1)$.
\end{proof}
We then follow the idea of ``optional stopping'' used in \citep{kakade2020information} and show the following ``optional stopping'' simulation lemma.
\begin{lemma}[``Optional Stopping'' Simulation Lemma]\label{lem:option-stop-simulation}
Consider the stochastic process over the trajectories $\{(s_h, a_h)\}_{h=0}^{H-1}$ generated via policy $\pi$ and ground truth $f^*$, where the randomness is from the Gaussian noise in the dynamics. Define a stopping time $\tau$ w.r.t this stochastic process and a given model $\hat{f}$ via:
\begin{align*}
    \tau := \min\{h\geq 0: \hat{V}_h^\pi(s_h) \leq V_h^\pi(s_h)\}.
\end{align*}
Furthermore, define a random variable:
\begin{align*}
    \tilde{V}_h^\pi(s_h) = \max\{\hat{V}_h^\pi(s_h), V_h^\pi(s_h)\},
\end{align*}
we have that
\begin{align*}
    \hat{V}_0^\pi(s_0) - V_0^\pi(s_0) \leq \mathbb{E}\left[\sum_{h=0}^{H-1}\mathbf{1}_{h<\tau} \left(\mathbb{E}_{s_{h+1}^\prime \sim \mathcal{N}(f^*(s_h, a_h), \sigma^2 I)}\tilde{V}_h^\pi(s_{h+1}^\prime) - \mathbb{E}_{s_{h+1}^\prime \sim \mathcal{N}(\hat{f}(s_h, a_h), \sigma^2 I)}\tilde{V}_h^\pi(s_{h+1}^\prime)\right)\right],
\end{align*}
where the expectation is w.r.t the stochastic process over the trajectories.
\end{lemma}
\begin{proof}
Define the filtration $\mathcal{F}_h :=\{\epsilon_i\}_{i=0}^{h-1}$, where $\epsilon_i$ is the noise that add to the dynamics at step $i$. Define
\begin{align*}
    M_h = \mathbb{E}[\hat{V}_0^\pi(s_0) - V_0 | \mathcal{F}_h],
\end{align*}
which is a Doob martingale with respect to $\mathcal{F}_i$ \citep{grimmett2020probability}. As $\tau\leq H$, by Doob's optional stopping theorem, we have that
\begin{align*}
    \mathbb{E}[\hat{V}_0^\pi(s_0) - V_0] = \mathbb{E}[M_{\tau}] = \mathbb{E}[\mathbb{E}[\hat{V}_0^\pi(s_0) - V_0|\mathcal{F}_{\tau}]].
\end{align*}
We then provide a bound for $M_{\tau}$. By Lemma \ref{lem:difference}, we have that
\begin{align*}
    M_{\tau} = & \mathbb{E}[\hat{V}_0^\pi(s_0) - V_0|\mathcal{F}_{\tau}]\\
    = & \mathbb{E}_{s_\tau^\prime\sim \mathcal{N}(\hat{f}(s_{\tau-1}, a_{\tau-1}), \sigma^2 I)} \left[\hat{V}_{\tau}^\pi(s_\tau^\prime)\right] - V_{\tau}^\pi(s_{\tau})\\
    & + \mathbb{E}_{s_h^\prime\sim \mathcal{N}(\hat{f}(s_{h-1}, a_{h-1}), \sigma^2 I)} \left[ \hat{V}_h^\pi(s_h^\prime)\right] -\sum_{h=1}^{\tau-1} \hat{V}_h^\pi(s_h)\\
    = & \sum_{h=1}^{\tau}\left(\mathbb{E}_{s_h^\prime\sim \mathcal{N}(\hat{f}(s_h, a_h), \sigma^2)} \left[\hat{V}_h^\pi(s_h^\prime)\right] - \tilde{V}_h^\pi(s_h)\right)\\
    \leq & \sum_{h=1}^{\tau}\left( \mathbb{E}_{s_h^\prime\sim \mathcal{N}(\hat{f}(s_h, a_h), \sigma^2 I)} \left[\tilde{V}_h^\pi(s_h)\right] - \tilde{V}_h^\pi(s_h)\right)\\
    = & \sum_{h=1}^{H} \mathbf{1}_{h\leq \tau}\left(\mathbb{E}_{s_h^\prime\sim \mathcal{N}(\hat{f}(s_h, a_h), \sigma^2 I)} \left[\tilde{V}_h^\pi(s_h)\right] - \tilde{V}_h^\pi(s_h)\right),
\end{align*}
where the third inequality follows the definition of $\tau$ (and thus $V_{\tau}^\pi(s_{\tau}) = \tilde{V}_{\tau}^\pi(s_\tau)$ and $\hat{V}_h^\pi(s_h) = \tilde{V}_h^\pi(s_h)$ for $h < \tau$.)

The proof is then concluded via the following observation:
\begin{align*}
    & \mathbb{E} \left[\mathbf{1}_{h\leq \tau}\left( \mathbb{E}_{s_h^\prime\sim \mathcal{N}(\hat{f}(s_h, a_h), \sigma^2 I)} \left[\tilde{V}_h^\pi(s_h)\right] - \tilde{V}_h^\pi(s_h)\right)\right]\\
    = & \mathbb{E}\left[\mathbb{E}\left[\mathbf{1}_{h\leq \tau}\left( \mathbb{E}_{s_h^\prime\sim \mathcal{N}(\hat{f}(s_h, a_h), \sigma^2 I)} \left[\tilde{V}_h^\pi(s_h)\right] - \tilde{V}_h^\pi(s_h)\right)\bigg|\mathcal{F}_{h-1}\right]\right]\\
    = & \mathbb{E}\left[\mathbb{E}\left[\mathbf{1}_{h-1<\tau}\left( \mathbb{E}_{s_h^\prime\sim \mathcal{N}(\hat{f}(s_h, a_h), \sigma^2 I)} \left[\tilde{V}_h^\pi(s_h)\right] - \tilde{V}_h^\pi(s_h)\right)\bigg|\mathcal{F}_{h-1}\right]\right]\\
    = & \mathbb{E}\left[\mathbf{1}_{h-1<\tau}\mathbb{E}\left[\left( \mathbb{E}_{s_h^\prime\sim \mathcal{N}(\hat{f}(s_h, a_h), \sigma^2 I)} \left[\tilde{V}_h^\pi(s_h)\right] - \tilde{V}_h^\pi(s_h)\right)\bigg|\mathcal{F}_{h-1}\right]\right]\\
    = & \mathbb{E}\left[\mathbf{1}_{h-1<\tau}\left(\mathbb{E}_{s_h^\prime\sim \mathcal{N}(\hat{f}(s_h, a_h), \sigma^2)} \left[\tilde{V}_h^\pi(s_h)\right] - \mathbb{E}_{s_h^\prime\sim \mathcal{N}(f^*(s_h, a_h), \sigma^2)} \left[\tilde{V}_h^\pi(s_h)\right)\right]\right],
\end{align*}
where the third equality is due to the fact that $\mathbf{1}_{h-1<\tau}$ is measurable under $\mathcal{F}_{h-1}$. 
\end{proof}
Before we finally provide the proof of Lemma \ref{lem:simulation}, we state the following lemma that bound the expectation under two isotropic Gaussian distribution with different mean:
\begin{lemma}[Difference of Expectation under Different Mean Isotropic Gaussian]
\label{lem:expectation_diff}$\forall$ (approximately measurable) positive function $g$, we have that
\begin{align*}
    \mathbb{E}_{z\sim\mathcal{N}(\mu_1, \sigma^2 I)} [g(z)] - \mathbb{E}_{z\sim\mathcal{N}(\mu_2, \sigma^2 I)} [g(z)] \leq \min\left\{\frac{\sqrt{2}\|\mu_1 - \mu_2\|}{\sigma}, 1\right\} \sqrt{\mathbb{E}_{z\sim \mathcal{N}(\mu_1, \sigma^2 I)}[g(z)^2]}
\end{align*}
\end{lemma}
\begin{proof}
\begin{align*}
     & \mathbb{E}_{z\sim\mathcal{N}(\mu_1, \sigma^2 I)} [g(z)] - \mathbb{E}_{z\sim\mathcal{N}(\mu_2, \sigma^2 I)} [g(z)] \\
     = & \mathbb{E}_{z\sim\mathcal{N}(\mu_1, \sigma^2 I)} \left[g(z)\left(1 - \exp\left(\frac{2(\mu_1 - \mu_2)^\top z + \|\mu_2\|^2 - \|\mu_1\|^2}{2\sigma^2}\right)\right)\right]\\
     \leq & \sqrt{\mathbb{E}_{z\sim\mathcal{N}(\mu_1, \sigma^2 I)}[g(z)^2]} \sqrt{\mathbb{E}_{z\sim\mathcal{N}(\mu_1, \sigma^2 I)} \left(1 - \exp\left(\frac{2(\mu_2 - \mu_1)^\top z - \|\mu_2\|^2 + \|\mu_1\|^2}{2\sigma^2}\right)\right)^2}
\end{align*}
We then calculate
\begin{align*}
    & \mathbb{E}_{z\sim\mathcal{N}(\mu_1, \sigma^2 I)} \left(1 - \exp\left(\frac{2(\mu_2 - \mu_1)^\top z - \|\mu_2\|^2 + \|\mu_1\|^2}{2\sigma^2}\right)\right)^2\\
    = & 1 - \frac{2}{\sqrt{2\pi} \sigma^{d/2}} \int \exp\left(\frac{-\|z - \mu_1\|_2^2 + 2(\mu_2 - \mu_1)^\top z - \|\mu_2\|^2 + \|\mu_1\|^2}{2\sigma^2}\right) dz \\
    & + \frac{1}{\sqrt{2\pi} \sigma^{d/2}} \int \exp\left(\frac{-\|z - \mu_1\|_2^2 + 4(\mu_2 - \mu_1)^\top z - 2\|\mu_2\|^2 + 2\|\mu_1\|^2}{2\sigma^2}\right) dz\\
    = & -1 + \frac{1}{\sqrt{2\pi} \sigma^{d/2}} \int \exp\left(\frac{-\|z - (2\mu_2 - \mu_1)\|_2^2 + 2\|\mu_2 - \mu_1\|_2^2}{2\sigma^2}\right) dz\\
    = & -1 + \exp\left(\frac{\|\mu_2 - \mu_1\|_2^2}{\sigma^2}\right).
\end{align*}
Also notice that, as $g$ is positive, a simple bound is that
\begin{align*}
     \mathbb{E}_{z\sim\mathcal{N}(\mu_1, \sigma^2 I)} [g(z)] - \mathbb{E}_{z\sim\mathcal{N}(\mu_2, \sigma^2 I)} [g(z)] \leq \mathbb{E}_{z\sim\mathcal{N}(\mu_1, \sigma^2 I)} [g(z)]\leq \sqrt{\mathbb{E}_{z\sim\mathcal{N}(\mu_1, \sigma^2 I)}[g(z)^2]}.
\end{align*}
Thus,
\begin{align*}
    \mathbb{E}_{z\sim\mathcal{N}(\mu_1, \sigma^2 I)} [g(z)] - \mathbb{E}_{z\sim\mathcal{N}(\mu_2, \sigma^2 I)} [g(z)] \leq \sqrt{\mathbb{E}_{z\sim\mathcal{N}(\mu_1, \sigma^2 I)}[g(z)^2]}\sqrt{\min\left\{\exp\left(\frac{\|\mu_2 - \mu_1\|_2^2}{\sigma^2}\right) - 1, 1\right\}}.
\end{align*}
Notice that, if $\|\mu_2 - \mu_1\| \geq \sigma$, then $\exp\left(\frac{\|\mu_2 - \mu_1\|_2^2}{\sigma^2}\right) - 1 \geq 1$. Meanwhile, when $x\in [0, 1]$, $\exp(x) \leq 1 + 2x$. Thus, 
\begin{align*}
    \sqrt{\min\left\{\exp\left(\frac{\|\mu_2 - \mu_1\|_2^2}{\sigma^2}\right) - 1, 1\right\}} \leq \sqrt{\min\left\{1 + \frac{2\|\mu_2 - \mu_1\|_2^2}{\sigma^2} - 1, 1\right\}} = \min\left\{\frac{2\|\mu_2 - \mu_1\|^2}{\sigma^2}, 1\right\},
\end{align*}
which finishes the proof.
\end{proof}
With Lemma \ref{lem:option-stop-simulation}, we have that
\begin{align*}
    & \hat{V}_0^\pi(s_0) - V_0^\pi(s_0)  \\
    \leq & \mathbb{E}\left[\mathbf{1}_{h-1<\tau}\left(\mathbb{E}_{s_h^\prime\sim \mathcal{N}(\hat{f}(s_h, a_h), \sigma^2)} \left[\tilde{V}_h^\pi(s_h)\right] - \mathbb{E}_{s_h^\prime\sim \mathcal{N}(f^*(s_h, a_h), \sigma^2)} \left[\tilde{V}_h^\pi(s_h)\right)\right]\right]\\
    \leq & \sum_{h=0}^{H-1}  \mathbb{E}\left[\sqrt{\mathbb{E}_{s_{h+1}^\prime\sim \mathcal{N}(\hat{f}(s_h, a_h), \sigma^2)}\left[\tilde{V}_h^\pi(s_{h+1}^\prime)^2\right]} \min\left\{\frac{\sqrt{2}\|f^*(s_h, a_h) - \hat{f}(s_h, a_h)_2\|}{\sigma}, 1\right\}\right]\\
    \leq & \sum_{h=0}^{H-1} \sqrt{\mathbb{E}\left[\mathbb{E}_{s_{h+1}^\prime\sim \mathcal{N}(\hat{f}(s_h, a_h), \sigma^2)}\left[\tilde{V}_h^\pi(s_{h+1}^\prime)^2\right]\right]} \sqrt{\mathbb{E}\left[\min\left\{\frac{2\|f^*(s_h, a_h) - \hat{f}(s_h, a_h)_2^2\|}{\sigma^2}, 1\right\}\right]}\\
    \leq & \sqrt{\mathbb{E}\left[\sum_{h=0}^{H-1}\mathbb{E}_{s_{h+1}^\prime\sim P(\cdot|f^*(s_h, a_h))}\left[\tilde{V}_h^\pi(s_{h+1}^\prime)^2\right]\right]} \sqrt{\mathbb{E}\left[\sum_{h=0}^{H-1} \min \left\{\frac{2\|f^*(s_h, a_h) - \hat{f}(s_h, a_h)_2^2\|}{\sigma^2}, 1\right\}\right]}\\
    \leq & H^{3/2}\sqrt{\mathbb{E}\left[\sum_{h=0}^{H-1} \min \left\{\frac{2\|f^*(s_h, a_h) - \hat{f}(s_h, a_h)_2^2\|}{\sigma^2}, 1\right\}\right]}
\end{align*}
where the second inequality is due to Lemma \ref{lem:expectation_diff}, and the last inequality is due to the fact that $\tilde{V}_h^\pi(s_{h+1}^\prime) \leq H$, $\forall h$.
\end{proof}
\subsection{Sum of Width Square}
\begin{lemma}[Bound on the Sum of Width Square] \label{lem:width_sum_bound}
Define 
\begin{align*}
    w_{\mathcal{F}}(s, a) := \sup_{\bar{f},\underline{f}\in\mathcal{F}}\|\bar{f}(s, a) - \underline{f}(s, a)\|_2.
\end{align*}
If $\{\beta_k^*\}_{k\in [K]}$ is a non-decreasing sequence, and $\|f\|_{2} < C, \forall f\in \mathcal{F}$, then:
\begin{align*}
    \sum_{k\in [K]} \sum_{h=0}^{H-1} w_{\mathcal{F}_t}^2(s_h^k, a_h^k) \leq 1 + 4C^2 H \mathrm{dim}_{E}\left(\mathcal{F}, T^{-1/2}\right) + 4\beta_K \mathrm{dim}_{E}\left(\mathcal{F}, T^{-1/2}\right) (1 + \log T)
\end{align*}
\end{lemma}
\label{sec:proof_width}
\begin{proof}
We first show the following lemma, which will be helpful in our proof.
\begin{lemma}[Lemma 1 in \citep{osband2014model}]
If $\{\beta_k\}_{k\in [K]}$ is a non-decreasing sequence, we have
\begin{align*}
    \sum_{k\in [K]}\sum_{h=0}^{H-1} \mathbf{1}_{w_{\mathcal{F}_{k}}(s_h^k, a_h^k) > \epsilon} \leq \left(\frac{4\beta_K}{\epsilon^2} +  H \right)\mathrm{dim}_{E}(\mathcal{F}, \epsilon).
\end{align*}
\end{lemma}
\begin{proof}
We first consider when $w_{\mathcal{F}_k}(s_h^k, a_h^k) > \epsilon$ and is $\epsilon$-dependent on $n$ disjoint sub-sequences of $\{(s_h^i, a_h^i)\}_{i\in [k-1]}$. By the definition of $\epsilon$-dependent, we know $\|\bar{f} - \underline{f}\|_{2, E_k} > n\epsilon^2$. On the other hand, by triangle inequality, we know $\|\bar{f} - \underline{f}\|_{2, E_k}\leq 2\sqrt{\beta_k} \leq 2\sqrt{\beta_K}$, thus $n < \frac{4\beta_K}{\epsilon^2}$. Hence we know when $w_{\mathcal{F}_k}(s_h^k, a_h^k) > \epsilon$, then $(s_h, a_h)$ is at most $\epsilon$-dependent on $\frac{4\beta_K}{\epsilon^2}$ disjoint sub-sequences of $\{(s_h^i, a_h^i)\}^{i\in [k-1]}$.

We then show that, for any sequence $\{(s_i, a_i)\}_{i\in [N]}$, there is some element $(s_j, a_j)$ that is $\epsilon$-dependent on at least $\frac{n}{\mathrm{dim}_E(\mathcal{F}, \epsilon)} - H$ disjoint sub-sequences of $\{(s_i, a_i)\}_{i\in [j-1]}$. Let $n$ satisfies that $n\mathrm{dim}_E(\mathcal{F}, \epsilon) + 1 \leq N \leq (n+1)\mathrm{dim}_E(\mathcal{F}, \epsilon) $, and we will construct $n$ disjoint sub-sequences $\{B_i\}_{i\in [n]}$. We first let $B_i = \{(s_i, a_i)\}, \forall i\in [n]$. If $(s_{k+1}, a_{k+1})$ is $\epsilon$-dependent on each $B_i, i\in [n]$, we have the desired results. Otherwise, we append $(s_{k+1}, a_{k+1})$ to the sub-sequence that it is $\epsilon$-independent with. Repeat this process until some $j > n + 1$ is $\epsilon$-dependent on each sub-sequence or we have reached $N$. In the latter case we have $\sum_{i\in [n]}|B_i| \geq n \mathrm{dim}_E(\mathcal{F}, \epsilon)$ (here we can add at most $H-1$ data to avoid the case we need a new episode of data), and since each element of a sub-sequence is $\epsilon$-independent with its predecessors, $|B_i|\leq \mathrm{dim}_E(\mathcal{F}, \epsilon), \forall i$ by the definition of eluder dimension. Thus $|B_i|=\mathrm{dim}_E(\mathcal{F}, \epsilon), \forall i$. And in this case, $(s_N, a_N)$ must be $\epsilon$-dependent on each sub-sequence by the definition of eluder dimension. Notice that, as our data is collected in an episodic pattern, there are at most $H-1$ sub-sequences that contains "imaginary" final episode data introduced to the construction. In this case, we know that there are at least $\frac{n}{\mathrm{dim}_E(\mathcal{F}, \epsilon)} - H$ disjoint sub-sequences that $(s_N, a_N)$ is $\epsilon$-dependent, which finishes our claim.

We finally consider the sub-sequence $B = \{(s_h^k, a_h^k)\}$ with $w_{\mathcal{F}_k}(s_h^k, a_h^k) > \epsilon$. We know that each element in $B$ is $\epsilon$-dependent on at most $\frac{4\beta_K}{\epsilon^2}$ disjoint sub-sequence of $B$, but at least $\epsilon$-dependent on $\frac{|B|}{\mathrm{dim}_E(\mathcal{F}, \epsilon)} - H$ sub-sequence of $B$. Thus we know $|B| \leq \left(\frac{4\beta_K}{\epsilon^2} +  H \right)\mathrm{dim}_{E}(\mathcal{F}, \epsilon)$, which concludes the proof.
\end{proof}

For notation simplicity, we define $w_{t, h} := w_{\mathcal{F}_t}(s_h^t, a_h^t)$. We first reorder the sequence $\{w_{t, h}\}_{k\in [K], 0\leq h\leq H-1}\to \{w_{i}\}_{i\in [KH]}$, such that $w_1\geq \cdots w_{TH}$. Then we have
\begin{align*}
    \sum_{k\in [K]} \sum_{h=0}^{H-1}w_{\mathcal{F}_t}^2(s_h^k, a_h^k) = \sum_{i\in [KH]} w_i^2 \leq \sum_{i\in [KH]} w_i^2 \mathbf{1}_{w_i < T^{-1/2}} + \sum_{i\in [KH]} w_i^2 \mathbf{1}_{w_i \geq T^{-1/2}}\leq 1 + \sum_{i\in [KH]} w_i^2 \mathbf{1}_{w_i \geq T^{-1/2}}.
\end{align*}
As we order the sequence, $w_j \geq \epsilon$ means
\begin{align*}
    \sum_{k\in [K]}\sum_{h=0}^{H-1}  \mathbf{1}_{w_{\mathcal{F}_{t}}(s_h^k, a_h^k) > \epsilon} \geq j.
\end{align*}
Hence we know
\begin{align*}
    \epsilon \leq \sqrt{\frac{4\beta_K}{\frac{j}{\mathrm{dim}_{E}(\mathcal{F}, \epsilon)} - H}} = \sqrt{\frac{4\beta_K \mathrm{dim}_{E}(\mathcal{F}, \epsilon)}{j - H\mathrm{dim}_{E}(\mathcal{F}, \epsilon)}},
\end{align*}
which means if $w_i \geq T^{-1/2}$, then $w_i < \min \left\{2C, \sqrt{\frac{4\beta_K \mathrm{dim}_{E}(\mathcal{F}, T^{-1/2})}{k - H\mathrm{dim}_{E}(\mathcal{F}, T^{-1/2})}}\right\}$. Hence,
\begin{align*}
    \sum_{i\in [KH]} w_i^2 \mathbf{1}_{w_i\geq T^{-1/2}} \leq & 4C^2 H \mathrm{dim}_{E}\left(\mathcal{F}, T^{-1/2}\right) + \sum_{j=H\mathrm{dim}_{E}(\mathcal{F}, T^{-1/2}) + 1}^T \frac{4\beta_K \mathrm{dim}_{E}(\mathcal{F}, T^{-1/2})}{j - H\mathrm{dim}_{E}(\mathcal{F}, T^{-1/2})}\\
    \leq & 4C^2 H \mathrm{dim}_{E}\left(\mathcal{F}, T^{-1/2}\right) + 4\beta_K \mathrm{dim}_{E}\left(\mathcal{F}, T^{-1/2}\right) (1 + \log T),
\end{align*} 
which finishes the proof.
\end{proof}
\subsection{Proof for Theorem \ref{thm:regret_bound} and Theorem \ref{thm:regret_bound_UCB}}
\label{sec:proof_regret}
\begin{proof}
Define $\mathcal{E}_k = \mathbb{P}_{f^*}\left(f^* \in \mathcal{F}_k\right)$. When constructing the confidence set, take $\alpha = T^{-1/2}$ and $\delta = 0.25$ in Lemma \ref{lem:confidence_set}, which leads to 
\begin{align*}
    \beta_k^* := 8\sigma^2 \log(4\mathcal{N}(\mathcal{F}, T^{-1/2}, \|\cdot\|_2)) + HT^{-1/2}(12C + \sqrt{8d\sigma^2 \log (16k^2 H)}).
\end{align*} 
With our confidence set construction, we know that $\sum_{k\in [K]}P(\bar{\mathcal{E}}_k)\leq 0.5$. Notice that
\begin{align*}
    \mathrm{Regret}(K) = &  \sum_{k\in [K]} \left[V_0^*(s_0^k) - V_0^{\pi_k}(s_0^k)\right]\\
    \leq &  \mathbb{E}\left[\sum_{k\in [K]}\mathbb{E}\left[\mathbb{P}(\mathcal{E}_k)[V^*(s_0^k) - V_0^{\pi_k}(s_0^k)]\right]\right] + H\sum_{k\in [K]}\mathbb{P}(\bar{\mathcal{E}}_k) \\
    \leq & \mathbb{E}\left[\sum_{k\in [K]}\mathbb{E}\left[\tilde{V}_{0, k}^{\pi_k}(s_0^k) - V_0^{\pi_k}(s_0^k)\right]\right] + 0.5 H \\
    \leq & H^{3/2} \sum_{k\in K}\sqrt{\mathbb{E} \left[\sum_{h=0}^{H-1}\min\left\{\frac{2\|\tilde{f}_k(s_h^k, a_h^k) - f^*(s_h^k, a_h^k) \|_2^2}{\sigma^2}, 1\right\}\right]} + 0.5H \\
    \leq & \sqrt{H^2T\mathbb{E}\left[\sum_{k\in [K]} \sum_{h=0}^{H-1}\min\left\{\frac{2\|\tilde{f}_k(s_h^k, a_h^k) -\hat{f}^*(s_h^k, a_h^k) \|_2^2}{\sigma^2}, 1\right\}\right]} + 0.5H \\
    \leq & \sqrt{\frac{2H^2 T}{\sigma^2} \left(1 + 4C^2 H \mathrm{dim}_{E}\left(\mathcal{F}, T^{-1/2}\right) + 4\beta_K^* \mathrm{dim}_{E}\left(\mathcal{F}, T^{-1/2}\right) (1 + \log T)\right)} + 0.5H,
\end{align*}
where the first equality is due to the fact that the total reward for each episode is bounded in $[0, H]$, the second inequality is due to the optimism and our confidence set construction, the third inequality is due to Lemma \ref{lem:simulation}, the fourth inequality is due to Cauchy-Schwartz inequality and the final inequality is due to Lemma \ref{lem:width_sum_bound} {\color{black}, which concludes the proof of Theorem \ref{thm:regret_bound_UCB}. Following the idea of \citep{russo2013eluder, russo2014learning, osband2014model}, we can translate the worst-case regret bound for UCB algorithm into the expected regret bound for TS algorithm, that conclude the proof of Theorem \ref{thm:regret_bound}.}
\end{proof}
{\color{black} \paragraph{Remark} It can be undesirable that our regret bound scale with $\sigma^{-1}$, which means our algorithm can perform pretty bad when the noise level is extremely low. It is also more or less counter-intuitive. We want to remark that, such phenomenon is only an artifact introduced by our proof strategy. The simulation lemma (Lemma \ref{lem:simulation}) works well when $f(s, a) - \tilde{f}(s, a)$ is small. However, we need to tolerate some bad episodes to collect sufficient samples, that can eventually make the error small. Fortunately, the regret of such bad episode is at most $H$. Hence, we can use the following strategy to get rid of the dependency on $\sigma^{-1}$.

\begin{definition}[Bad and Good Episodes] Define episode $k$ as a bad episode, if $\exists h \in \{0, 1, \cdots, H-1\}$, such that $w_{k, h} := w_{\mathcal{F}_k}(s_h^k, a_h^k)$ is the largest $H \mathrm{dim}_{E}(\mathcal{F}, \sigma^2 T^{1/2})$ elements in the set $\{w_{k, h}\}_{k\in[K], 0\leq h \leq H-1}$. Define episode $k$ as a good episode, if it is not a bad episode.
\end{definition}

By the definition, we know there are at most $H \mathrm{dim}_{E}(\mathcal{F}, \sigma^2 T^{-1/2})$ bad episodes. We then show the following lemma, that can be directly generalized from Lemma \ref{lem:width_sum_bound}, by setting $\epsilon = \sigma^2 T^{-1/2}$ and remove the terms from bad episodes.
\begin{lemma}
\label{lem:width_sum_bound_good}
If $\{\beta_k^*\}_{k\in [K]}$ is a non-decreasing sequence, and $\|f\|_{2} < C, \forall f\in \mathcal{F}$, then:
\begin{align*}
    \sum_{k\in [K], \text{$k$ is good}} \sum_{h=0}^{H-1} w_{\mathcal{F}_t}^2(s_h^k, a_h^k) \leq \sigma^2 + 4\beta_K \mathrm{dim}_{E}\left(\mathcal{F}, \sigma^2 T^{-1/2}\right) (1 + \log T)
\end{align*}
\end{lemma}

Eventually, we can obtain the following regret bound, by setting the regret of bad episodes as $H$, and bounding the regret of good episodes with Lemma \ref{lem:width_sum_bound_good}.

\begin{theorem}[Improved Regret Bound]
\label{thm:regret_bound_UCB_improved}
Assume Assumption \ref{assump:bounded_output} to \ref{assump:bounded_eluder} holds. Take $\alpha = \sigma^2 T^{-1/2}$ and $\delta = 0.25$ in Lemma \ref{lem:confidence_set}, which leads to 
\begin{align*}
    \beta_k^* := 8\sigma^2 \log(4\mathcal{N}(\mathcal{F}, \sigma^2 T^{-1/2}, \|\cdot\|_2)) + H\sigma^2 T^{-1/2}(12C + \sqrt{8d\sigma^2 \log (16k^2 H)}).
\end{align*}
We have that
\begin{align*}
\textstyle
    \mathrm{Regret}(K) \leq \sqrt{H^2 T (\frac{8\beta_K}{\sigma^2} + 1)\mathrm{dim}_{E}(\mathcal{F}, \sigma^2 T^{-1/2})(1 + \log T)} + 0.5 H + H^2 \mathrm{dim}_{E}(\mathcal{F}, \sigma^2 T^{-1/2})
\end{align*}
\end{theorem}
We would like to remark, that the definition of bad and good episodes is only used for the proof. We don't need to make any modification on the algorithm. Notice that, as $\beta_k^* \propto \sigma^2$, our upper bound in Theorem \ref{thm:regret_bound_UCB_improved} can only scale with $\sigma^{-1}$ through the logarithm covering number $\log(4\mathcal{N}(\mathcal{F}, \sigma^2 T^{-1/2}, \|\cdot\|_2))$ and eluder dimension $\mathrm{dim}_{E}(\mathcal{F}, \sigma^2 T^{-1/2})$. When $\mathcal{F}$ is a linear function class, both term should scale with $\mathrm{polylog}(\sigma)$, that matches the result from \citep{kakade2020information}. 
}

%% file: tex/appendix_complexity.tex
\section{Bounds on the Complexity Term under Linear Realizability}
\label{sec:linear_case}
We provide the upper bound on the covering number and the eluder dimension of $\mathcal{F}$ when $\mathcal{F}:= \{\theta^\top \varphi: \theta\in\mathbb{R}^{d_{\varphi}\times d}, \|\theta\|_2 \leq W\}$ where $\varphi:\mathcal{S}\times\mathcal{A} \to \mathbb{R}^{d_{\varphi}}$ is some known feature map. We first make the following standard assumption:
\begin{assumption}[Bounded Feature]
\begin{align*}
    \|\varphi(s, a)\|_2 \leq B, \forall (s, a)\in\mathcal{S}\times\mathcal{A}.
\end{align*}
\end{assumption}
\subsection{Covering Number}
\begin{theorem}[Covering Number Bound]
We have that
\begin{align*}
    \mathcal{N}(\mathcal{F}, \epsilon, \|\cdot\|_2) \leq \left(1 + \frac{2BW}{\epsilon}\right)^{d_{\varphi}}.
\end{align*}
\end{theorem}
\begin{proof}
Notice that, by Cauchy-Schwartz inequality, we have that
\begin{align*}
    \max_{(s, a)\in\mathcal{S}\times\mathcal{A}} \|\varepsilon_i^\top \varphi(s, a)\|_2 \leq B \|\varepsilon_i\|_2, \quad \forall \varepsilon_i\in\mathbb{R}^{d_{\varphi}}.
\end{align*}
Thus, denote $\varepsilon = [\varepsilon_i]_{i\in [d]}$, we have that
\begin{align*}
    \max_{(s, a)\in\mathcal{S}\times\mathcal{A}} \|\varepsilon^\top \varphi(s, a)\|_2^2 = \max_{(s, a)\in\mathcal{S}\times\mathcal{A}} \sum_{i\in [d]}\|\varepsilon_i^\top \varphi(s, a)\|_2^2\leq B^2 \sum_{i\in [d]}\|\varepsilon_i\|_2^2 = B^2 \|\varepsilon\|_2^2.
\end{align*}
Hence, to find an $\epsilon$-cover for $\mathcal{F}$, we just need to find an $\epsilon/B$-cover of $\{\theta:\theta\in\mathbb{R}^{d_{\varphi}\times d}, \|\theta\|_2\leq W\}$. By standard argument on the covering number of Euclidean space (e.g. Lemma 5.7 in \citep{wainwright2019high}), we can conclude the desired result.
\end{proof}
\subsection{Eluder Dimension}
\begin{theorem}[Eluder Dimension Bound]
We have that
\begin{align*}
    \mathrm{dim}_{E}(\mathcal{F}, \epsilon) \leq \frac{3d_{\varphi} e}{e-1}\log \left(3 + \frac{12W^2 B^2}{\epsilon^2}\right) + 1.
\end{align*}
\end{theorem}
\begin{proof}
Our proof follows the idea in \citep{russo2013eluder}. Define
\begin{align*}
    w_k := \sup\left\{(\theta_1 - \theta_2)^\top \varphi(s, a):\sqrt{\sum_{i\in[k-1]}\left((\theta_1 - \theta_2)^\top \varphi_i(s_i, a_i)\right)^2 }\leq \epsilon^\prime, \theta_1, \theta_2 \in \mathbb{R}^{d_{\varphi\times d}}, \|\theta_1\|\leq W, \|\theta_2\|\leq W\right\}.
\end{align*}
For notation simplicity, define $\varphi_k := \varphi(s_i, a_i)$, $\theta := \theta_1 - \theta_2$, and $\Phi_k := \sum_{i\in[k-1]}\varphi_i\varphi_i^\top$. Obviously, we have that $\|\theta\| \leq 2W$. Moreover, by straightforward calculation, we know
\begin{align*}
    \sum_{i\in [k-1]}\left((\theta_1 - \theta_2)^\top \varphi_i(s_i, a_i)\right)^2 = \mathrm{Trace}(\theta^\top \varphi_k \theta).
\end{align*}
Define $V_k:= \Phi_k + \frac{(\epsilon^\prime)^2}{4W^2} I$, we start from considering the problem
\begin{align*}
    \max_{\theta} \mathrm{Trace}(\theta^\top \varphi_k \varphi_k^\top \theta), \quad \text{subject to}\quad \mathrm{Trace}(\theta^\top V_k  \theta) \leq 2\epsilon^2.
\end{align*}
The Lagrangian can be formed as
\begin{align*}
    \mathcal{L}(\theta, \gamma) = - \mathrm{Trace}(\theta^\top \varphi_k \varphi_k^\top \theta) + \lambda (\mathrm{Trace}(\theta^\top V_k\theta) - 2\epsilon^2), \quad \lambda \geq 0.
\end{align*}
The optimality condition of $\theta$ is
\begin{align*}
    (\lambda V_k - \varphi_k\varphi_k^\top) \theta = 0.
\end{align*}
As $V_k$ is of full rank, $\lambda V_k - \varphi_k\varphi_k^\top$ has rank at least $d_{\varphi} - 1$ (as $\varphi_k \varphi_k^\top$ is of rank $1$). So the equation
\begin{align*}
     (\lambda V_k - \varphi_k\varphi_k^\top) \theta_i = 0, \quad \theta_i \in \mathbb{R}^{d_{\varphi}}
\end{align*}
only has one non-zero solution. Substitute back, we know that (define $\|x\|_{A} := \sqrt{x^\top A x}$):
\begin{align*}
    \sup\{\mathrm{Trace}(\theta^\top \varphi_k \varphi_k^\top \theta):  \mathrm{Trace}(\theta^\top V_k  \theta) \leq \epsilon^2\} = \sqrt{2}\epsilon^\prime \|\varphi_k\|_{V_k^{-1}}.
\end{align*}
With the conclusion above, we have that
\begin{align*}
    w_k \leq \sup\{\theta^\top \varphi_k: \mathrm{Trace}(\theta^\top \Phi_k  \theta) \leq \epsilon^2, \|\theta\|\leq 2W\}\leq \sup\{\theta^\top \varphi_k:  \mathrm{Trace}(\theta^\top V_k  \theta) \leq 2\epsilon^2\} = \sqrt{2}\epsilon^\prime \|\varphi_k\|_{V_k^{-1}}.
\end{align*}
Hence, if $w_{k}\geq \epsilon^\prime$, then $\varphi_k V_k^{-1}\varphi_k \geq 0.5$. Moreover, with Matrix Determinant Lemma, if $w_i\geq \epsilon^\prime$, $\forall i<k$, we have 
\begin{align*}
    \mathrm{det}(V_k) = \mathrm{det}(V_{k-1})(1 + \varphi_k^\top V_{k}^{-1} \varphi_k) \geq  \mathrm{det}(V_{k-1})\left(\frac{3}{2}\right)\geq \cdots \geq  \mathrm{det}\left(\frac{(\epsilon^\prime)^2}{4W^2} I\right) \left(\frac{3}{2}\right)^{k-1} = \frac{(\epsilon^\prime)^{2d}}{4W^{2d}}  \left(\frac{3}{2}\right)^{k-1}.
\end{align*}
Meanwhile,
\begin{align*}
    \mathrm{det}(V_k) \leq \left(\frac{\mathrm{Trace}(V_k)}{d}\right)^d \leq \left(\frac{B^2(k-1)}{d} + \frac{(\epsilon^\prime)^2}{4W^2}\right)^d.
\end{align*}
Hence, we know
\begin{align*}
    \left(\frac{3}{2}\right)^{(k-1)/d} \leq \frac{4W^2B^2}{(\epsilon^\prime)^2}\cdot \frac{k-1}{d} + 1.
\end{align*}
Now we only need to find the largest $k$ that can make this inequality hold. For notation simplicity, define $\alpha := \frac{4W^2B^2}{(\epsilon^\prime)^2}$, $n = \frac{k-1}{d}$. As $\log(1+x) \geq \frac{x}{1+x}$ and $\log x \leq x/e$, we have
\begin{align*}
    \frac{n}{3}\leq n \log 3/2 \leq \log(\alpha + 1) + \log n \leq \log(\alpha + 1) + \log 3 + \log (n/3)\leq \log(\alpha + 1) + \log 3 + \frac{n}{3e}.
\end{align*}
Substitute back, we can obtain the desired result.
\end{proof}

%% file: tex/appendix_experiment.tex
\newpage
\section{Experimental Details}
\label{sec:exp_details}
\subsection{Algorithm Summary}
Our algorithm is easily built on SAC. The only difference we make is we decouple the critic network into a representation network $\phi(\cdot)$ and a linear layer $l(\cdot)$ on top of the representation. The representation network is governed by the model dynamics loss in \algabb, and we train a linear layer to predict the $Q$-value as it lies in the linear space of the representation guaranteed by our analysis. We update the representation by a momentum factor and keep the policy update the same procedure as SAC.
\subsection{Full Experiments}\label{appendix:full_exp}
\begin{table*}[h]
\caption{Performance of \algabb on various MuJoCo control suite tasks. Our method achieve strong performance even comparing to pure empirical baselines. To be specific, in hard tasks like Humanoid-ET and Ant-ET, \algabb outperforms the baselines significantly. Results with $^*$ are directly adopted from MBBL~\citep{wang2019benchmarking}. We also provide the SoTA model-free RL method SAC as a reference.}
\vspace{0.3em}
\footnotesize
\setlength\tabcolsep{3.5pt}
\label{tab:MuJoCo_results_full}
\centering
\begin{tabular}{p{2.5cm}p{2.5cm}p{2.5cm}p{2.5cm}p{2.5cm}p{2.5cm}p{2cm}}
\toprule
& Swimmer & Ant-ET & Hopper-ET & Pendulum \\ 
\midrule  
ME-TRPO$^*$ & 30.1$\pm$9.7 & 42.6$\pm$21.1 & 4.9$\pm$4.0 & 177.3$\pm$1.9\\
PETS-RS$^*$  & 42.1$\pm$20.2 & 130.0$\pm$148.1 & 205.8$\pm$36.5 & 167.9$\pm$35.8\\
PETS-CEM$^*$  & 22.1$\pm$25.2 & 81.6$\pm$145.8 & 129.3$\pm$36.0 & 167.4$\pm$53.0\\
DeepSF & 25.5$\pm$13.5 & 768.1$\pm$44.1 & 548.9$\pm$253.3 & 168.6$\pm$5.1 \\
{\bf \algabb} & 42.6$\pm$4.2 & 806.2$\pm$60.2 & 732.2$\pm$263.9 & 169.5$\pm$0.6 \\
\midrule
SAC$^*$  & 41.2$\pm$4.6 & 2012.7$\pm$571.3 & 1815.5$\pm$655.1 & 168.2$\pm$9.5 \\
\bottomrule 
\end{tabular}
\centering
\begin{tabular}{p{2.5cm}p{2.5cm}p{2.5cm}p{2.5cm}p{2.5cm}p{2.5cm}p{2cm}}
\toprule
& Reacher & Cartpole & I-pendulum & Walker-ET \\ 
\midrule  
ME-TRPO$^*$ & -13.4$\pm$5.2 & 160.1$\pm$69.1 & -126.2$\pm$86.6 & -9.5$\pm$4.6\\
PETS-RS$^*$ & -40.1$\pm$6.9 &  195.0$\pm$28.0 & -12.1$\pm$25.1 & -0.8$\pm$3.2 \\
PETS-CEM$^*$ & -12.3$\pm$5.2 &  199.5$\pm$3.0 & -20.5$\pm$28.9 & -2.5$\pm$6.8 \\
DeepSF & -16.8$\pm$3.6 & 194.5$\pm$5.8 & -0.2$\pm$0.3 & 165.6$\pm$127.9\\
{\bf \algabb} & -7.2$\pm$1.1 & 138.2$\pm$39.5 & 0.0$\pm$0.0 & 501.58$\pm$204.0  \\
\midrule
SAC$^*$ & -6.4$\pm$0.5 & 199.4$\pm$0.4 & -0.2$\pm$0.1 & 2216.4$\pm$678.7\\
\bottomrule 
\end{tabular}
\centering
\begin{tabular}{p{2.5cm}p{2.5cm}p{2.5cm}p{2.5cm}p{2.5cm}p{2.5cm}p{2cm}}
\toprule
& MountainCar & Acrobot & SlimHumanoid-ET & Humanoid-ET \\ 
\midrule  
ME-TRPO$^*$  & -42.5$\pm$26.6 & 68.1$\pm$6.7 & 76.1$\pm$8.8 & 776.8$\pm$62.9\\
PETS-RS$^*$ & -78.5$\pm$2.1 & -71.5$\pm$44.6 & 320.7$\pm$182.2 & 106.9$\pm$102.6\\
PETS-CEM$^*$  & -57.9$\pm$3.6 & 12.5$\pm$29.0 & 355.1$\pm$157.1 & 110.8$\pm$91.0 \\
DeepSF & -17.0$\pm$23.4 & -74.4$\pm$3.2 & 533.8$\pm$154.9 & 241.1$\pm$116.6 \\
{\bf \algabb}  & 50.3$\pm$1.1 & -69.0$\pm$3.3 & 986.4$\pm$154.7 & 886.9$\pm$95.2 \\
\midrule
SAC$^*$  & 52.6$\pm$0.6 & -52.9$\pm$2.0 & 843.6$\pm$313.1 & 1794.4$\pm$458.3  \\
\bottomrule 
\end{tabular}
\end{table*}
\newpage

\subsection{Ablations}\label{appendix:ablations}
\begin{table*}[h]
\caption{Ablation Suty of \algabb on MuJoCo tasks. We see that a small momentum factor help stabilize the performance, especially in environments like  Huamoid and Hopper-ET.}
\footnotesize
\setlength\tabcolsep{3.5pt}
\label{tab:ablation}
\centering
\begin{tabular}{p{2.5cm}p{2.5cm}p{2.5cm}p{2.5cm}p{2.5cm}p{2cm}p{2cm}}
\toprule
& Hopper-ET & Ant-ET & S-Humanoid-ET & Humanoid-ET \\ 
\midrule  
\algabb-0.9  & 593.2$\pm$37.4 & \textbf{877.7$\pm$45.9} & 881.6$\pm$385.2 & 232.9$\pm$63.4 \\
\algabb-0.99 & 305.9$\pm$13.4 & 707.9$\pm$51.1 & 629.3$\pm$106.9 & 818.1$\pm$130.6 \\
\algabb-0.999 & \textbf{732.2$\pm$263.9} & 806.2$\pm$60.2 & \textbf{986.4$\pm$154.7} & \textbf{886.9$\pm$95.2} \\
\bottomrule 
\end{tabular}
\end{table*}
\paragraph{Momentum Update} Our ablation experiments are trying to study an important design choice of the practical algorithm: the momentum used to update the critic function. We summarize the results in Table~\ref{tab:ablation}. We can see that using a small large momentum factor such as 0.999 shows better performance. This is intuitively understandable: large momentum factor slows down the update speed of the representation of the critic function and thus stabilize the training. Such phenomenon illustrates the importance of slowly update the representation.

\begin{wrapfigure}[11]{R}{0.4\textwidth}
\centering
\vspace{-9mm}
\includegraphics[width=0.9\linewidth]{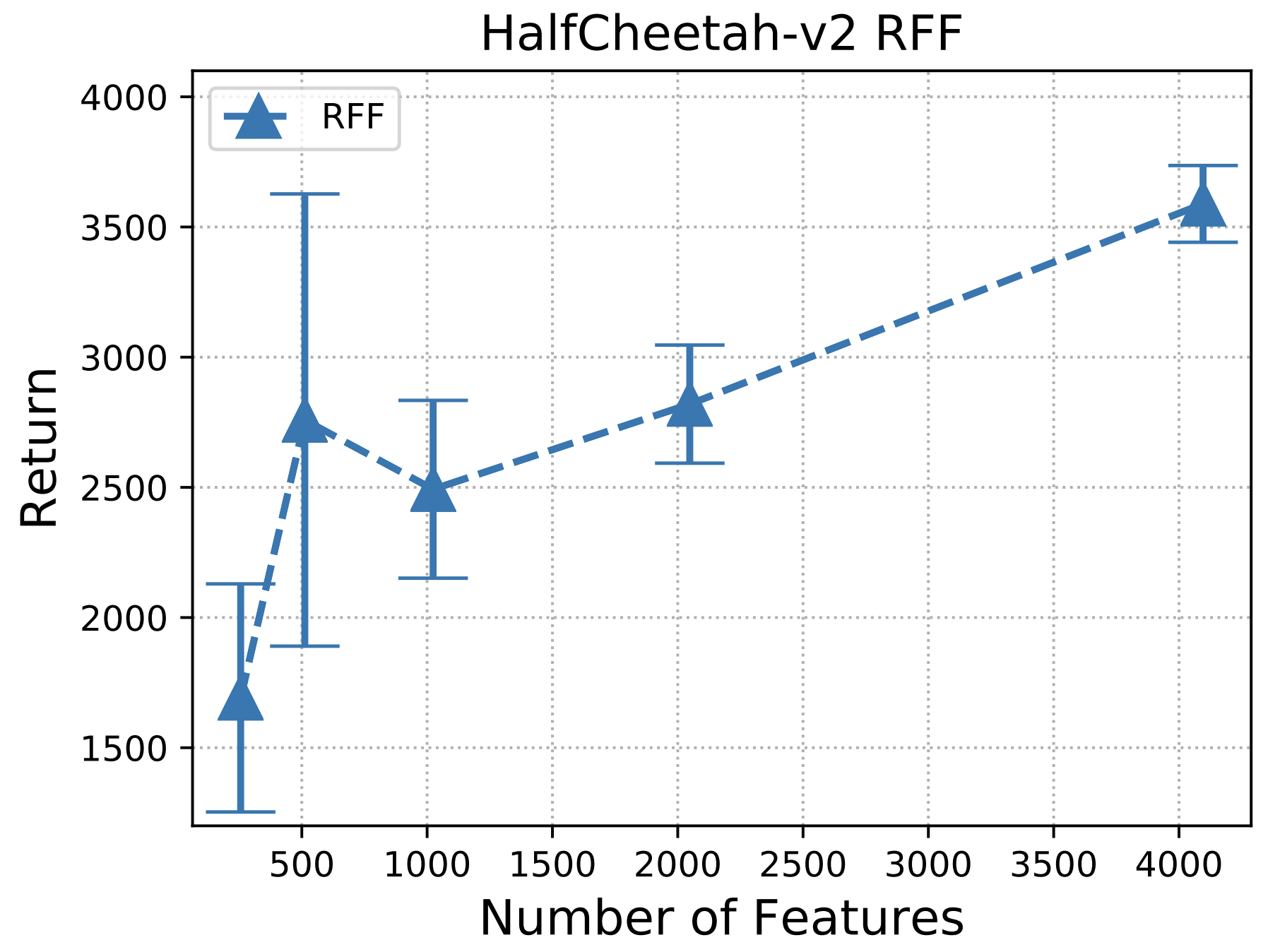}
\vspace{-2mm}
\caption{\small 
Increasing the number of random features can also lead to a performance gain.}
\label{fig:random_dim}
\end{wrapfigure}

\paragraph{Random Feature Dimension} We also conduct the experiments on how does the random feature dimension affect the final performance of the algorithm. We plot the results in HalfCheetah environment in Fig.~\ref{fig:random_dim}. We can see that when we increasing the random feature dimension, we see a performance gain on the final return. This suggests that using a larger number of feature dimension would help the performance. 

\paragraph{MLP Network for Critic Network} We also conduct an experiment to study whether adding a MLP network on top of our representation could work. We show such ablation in Tab.~\ref{tab:mlp}. From the results, we see that the performance of MLP network is in generally better than the Linear network.

\begin{table*}[h]
\centering
{\color{black} 
\caption{Comparison of \algabb linear critic network and critic network. Results show that in general MLP network will further improve the performance.}
\footnotesize
\setlength\tabcolsep{3.5pt}
\label{tab:mlp}
\centering
\begin{tabular}{p{2.5cm}p{2.5cm}p{2.5cm}p{2.5cm}p{2.5cm}p{2cm}p{2cm}}
\toprule
& Reacher & MountainCar & Cartpole & Acrobot\\ 
\midrule  
\algabb-Linear  & -7.2$\pm$1.1 & 50.3$\pm$1.1 & 138.2$\pm$39.5 & -69.0$\pm$3.3\\
\algabb-MLP & -6.8$\pm$0.4 & \textbf{53.8}$\pm$\textbf{1.1} & 171.9$\pm$31.0 & \textbf{-15.6}$\pm$\textbf{1.9}\\
SAC & \textbf{-6.4}$\pm$\textbf{0.5} & 52.6$\pm$0.6 & \textbf{199.4}$\pm$\textbf{0.4} & -52.9$\pm$2.0\\
\bottomrule 
\end{tabular}}

\setlength\tabcolsep{3.5pt}
\begin{tabular}{p{2.5cm}p{2.5cm}p{2.5cm}p{2.5cm}p{2.5cm}p{2cm}p{2cm}}
\toprule
& Pendulum & I-Pendulum & Walker-ET & S-Humanoid-ET\\ 
\midrule  
\algabb-Linear  & \textbf{169.5}$\pm$\textbf{0.6} & \textbf{0.0}$\pm$\textbf{0.0} & 501.6$\pm$204.0 & 986.4$\pm$154.7\\
\algabb-MLP & 165.9$\pm$4.2 & \textbf{0.0}$\pm$\textbf{0.0} & 1005.7$\pm$458.4 & \textbf{2521.1}$\pm$\textbf{420.8}\\
SAC & 168.2$\pm$9.5 & -0.2$\pm$0.1 & \textbf{2216.4}$\pm$\textbf{678.7} & 843.6$\pm$313.1\\
\bottomrule 
\end{tabular}
\end{table*}

{\color{black} 
\subsection{Comparison to LC3}\label{appendix:lc3}
We provide a comparison of empirical results with LC3~\citep{kakade2020information}, which is also an algorithm with rigorous theoretical guarantees. Despite the major difference that we are learning the representation while LC3 assumes a given feature, the performance of \algabb is much better than LC3 in tasks like Mountain Car and Hopper.}

\begin{table*}[h]
\centering
{\color{black} 
\caption{Comparison of \algabb with LC3 on MuJoCo tasks. LC3 only achieves good performance on relatively easy tasks like Reacher. However, their performance on Hopper and Mountain-Car is much worse than \algabb. }
\footnotesize
\setlength\tabcolsep{3.5pt}
\label{tab:lc3}
\centering
\begin{tabular}{p{2.5cm}p{2.5cm}p{2.5cm}p{2.5cm}p{2.5cm}p{2cm}p{2cm}}
\toprule
& Reacher & MountainCar & Hopper \\ 
\midrule  
\algabb  & -7.2$\pm$1.1 & \textbf{50.3}$\pm$\textbf{1.1} & \textbf{732.2}$\pm$\textbf{263.9} \\
LC3 & \textbf{-4.1}$\pm$\textbf{1.6} & 27.3$\pm$8.1 & -1016.5$\pm$607.4 \\
\bottomrule 
\end{tabular}}
\end{table*}

\newpage
\subsection{Performance Curves}
We provide an additional performance curve including ME-TRPO in Figure \ref{fig:MuJoCo1} for a reference.
\begin{figure*}[h]
    \centering
    \includegraphics[width=0.99\textwidth]{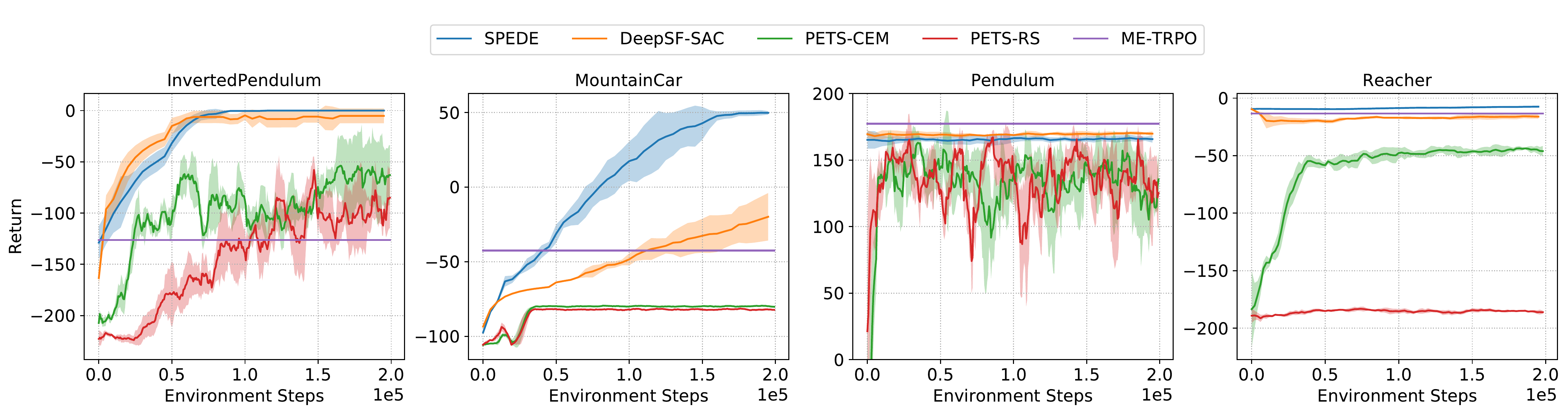}
    \caption{\footnotesize \textbf{Experiments on MuJoCo:} We show curves of the return versus the training steps for \algabb and model-based RL baselines. We also include the final performance of ME-TRPO from ~\citep{wang2019benchmarking} for reference.}
    \label{fig:MuJoCo1}
\end{figure*}

\subsection{Hyperparameters}
\label{appendix:hyperparam}
We conclude the hyperparameter we use in our experiments in the following.
\begin{table*}[h]
\caption{Hyperparameters used for \algabb in all the environments in MuJoCo.}
\footnotesize
\setlength\tabcolsep{3.5pt}
\label{tab:hyper}
\centering
\begin{tabular}{p{5cm}p{3cm}p{5cm}p{2.5cm}p{2.5cm}p{2cm}p{2cm}}
\toprule
& Hyperparameter Value \\ 
\midrule
Actor lr & 0.0003 \\
Model lr & 0.0001 \\
Actor Network Size & (1024, 1024, 1024) \\
Fourier Feature Size & 1024 \\
Discount & 0.99\\
Target Update Tau & 0.005 \\
Model Update Tau & 0.001 \\
Batch Size & 256 \\
\bottomrule 
\end{tabular}
\end{table*}